\documentclass[11pt]{article}
\usepackage{style}
\usepackage[left=1in,bottom=1in,top=1in,right=1in]{geometry}
\hypersetup{%
	colorlinks   = true, %
	urlcolor     = blue, %
	linkcolor    = blue, %
	citecolor    = blue,  %
}

\title{ Globally Consistent Coloring Schemes for Language Identification}

\author{Moses Charikar\thanks{Stanford University. Email: \texttt{moses@cs.stanford.edu.}}
\and
Jon Kleinberg\thanks{Cornell University. Email: \texttt{kleinberg@cornell.edu.}}
\and
Chirag Pabbaraju\thanks{Stanford University. Email: \texttt{cpabbara@cs.stanford.edu.}}
}

\date{\today}

\usepackage[dvipsnames]{xcolor}

\begin{document}

\maketitle

\begin{abstract}
We study how little extra information is needed to make adversarial language learning possible. 
In Gold's model of language identification in the limit, a learner is given an enumeration of the strings from an unknown target language chosen from a countable list of candidate languages.
The learner guesses the identity of the language over the course of the enumeration, and it succeeds if, after some finite time, all of its guesses are the correct language. 
Classical results of Gold and Angluin show that many natural collections cannot be learned in this way. 
Recent work on {\em trace colorings}, motivated by the success of annotation and thinking-trace strategies in language learning, overcomes this obstruction by annotating every symbol of every observed string with a color. 
We ask whether the learner really needs this whole sequence of colors, or whether one color at the end of each string (a terminal coloring) is enough to facilitate language identification.

We show that just one terminal bit per string is enough for every countable collection of infinite languages. 
In fact, the colorings can be chosen collection-independently: there is a single assignment of a two-color terminal coloring to every infinite language such that the same preassigned colorings identify every countable subcollection. 
Thus, in this model, an entire color trace can be compressed to one bit attached to the end
of each example.

This optimal compression has a set-theoretic cost. 
Our global construction uses almost-disjoint families and a transfinite recursion, and we prove that this kind of strong nonconstructivity is unavoidable for any bounded number of colors.
As a notion of constructivity, we use the standard formalism of {\em Borel maps} (a regularity condition satisfied by natural explicit constructions); 
we show that no global terminal coloring with a finite number of colors defined by a Borel map
can identify all countable subcollections. 
By contrast, known trace-coloring constructions are Borel when encoded as terminal colorings, but require infinitely many terminal colors. 
Our lower bound uses the Galvin-Prikry theorem (an infinite-dimensional generalization of Ramsey's Theorem) to canonicalize any Borel coloring into a finite-prefix coloring, and then proves that finite-prefix colorings cannot support identification for all countable collections.
These results give a sharp tradeoff between where the annotation is placed, how many colors it uses, and how explicitly the coloring rule can be defined.
\end{abstract}

\newpage

\section{Introduction}

Dramatic advances in the power of large language models 
have rewakened interest in abstract models of language learning.
These models can be viewed as 
formalisms for language acquisition that operate at a level of abstraction
more general than any particular architecture or algorithm,
and in this way they seek principles that might exist outside
the specific assumptions of any architectural
or algorithmic decisions we have made in current systems.

One of the very earliest and most influential of these formalisms
is Gold's notion of {\em language identification in the limit}
\cite{gold1967language},
which presents language learning as a game played between 
an adversary and an algorithm.
The adversary thinks of a secret language $K$ that is known to
come from a countable collection of candidate languages 
$L_1, L_2, L_3, \ldots$;
the adversary presents the strings of $K$ to the algorithm one-by-one;
and in each step the algorithm tries to guess the identity of the language.
(As usual, a {\em language} is simply a set of finite strings,
and we will generally think of all the languages in the collection
as infinite sets unless noted otherwise.)
In particular, in each step $t$ the adversary presents a string $w_t$ from 
the secret language $K$, and 
the algorithm guesses an index $i_t$, with the goal that $L_{i_t} = K$.
The algorithm wins this game if there is some $t^*$ such that 
$L_{i_t} = K$ for all $t \geq t^*$, and if it does this, then
it has achieved {\em identification in the limit}.

Given that the model is based on worst-case assumptions about 
the adversary's choice of language $K$ and its enumeration of $K$,
the game is very hard for the algorithm to win, and 
Gold's original theorem showed that identification in the limit was
not possible in general even for language collections
$L_1, L_2, L_3, \ldots$ as simple as the regular languages
\cite{gold1967language},
Angluin subsequently characterized the very limited cases in which the 
algorithm can succeed at this task
\cite{angluin1980inductive}.

\paragraph{Trace colorings.}
Because language learning appears to be more tractable in practice
than the bleak view suggested by Gold's negative result, 
attention turned to strengthenings
of the underlying assumptions that could produce formalisms 
reflecting this tractability.
Arguably the most widespread of these strengthenings was the
introduction of probabilistic assumptions: the language choice may be
adversarial, but the enumeration of the examples from the language
comes from an underlying distribution that therefore constrains
the adversary.

A recent line of work, in contrast, has explored a set of assumptions
that enhance the tractability of the model while maintaining its
core worst-case style.
This line of work is based on the observation that in practice,
training data for language learning often comes {\em annotated}
with additional meta-data or trace data that helps further explain 
the training data's membership in its underlying language
\cite{bhattamishra2026automata,papazov25learning,peng2026language}.
For example, training datasets consisting of computer code often
comes with comments in the code; training datasets consisting of
solutions to math problems often contain the written step-by-step
reasoning of expanded solution sets.
In such cases, the expanded meta-data or annotations have 
proved very helpful for training.
Indeed, guided in part by this insight, 
recent developments in large language models have greatly extended
the set of annotations via chain-of-thought mechanisms that explicitly
add thinking tokens to the output of the model
\cite{wei2022chain}.

All of these annotations have been viewed, 
in this recent line of theoretical work,
as {\em computational traces}: extra information that accompanies
the strings of the language as they are presented to the learner \cite{bhattamishra2026automata,papazov25learning,peng2026language}.
The idea in the theoretical models arising from this research
is the following:
we attach an annotation symbol to each symbol of each string,
and we present this to the algorithm as part of the enumeration
of strings in a language learning problem.
In this way, the language learning works with these 
``annotated strings'' --- where each symbol has an annotation added to it ---
rather than the raw strings themselves.

In the work of Peng et al.~(2026), these traces come from an underlying
machine model, in that we assume the language being learned is
the set of inputs recognized by a state machine (a finite automaton,
a pushdown automaton, or a Turing machine), and the annotating trace is
the state sequence of the accepting path
\cite{peng2026language}.
In the work of Bhattamishra et al.~(2026), the traces are more expansive, 
drawn from the power-set of the vocabulary
\cite{bhattamishra2026automata}.
More recently, Charikar et al.~(2026) abstracted these types of computational
traces away from explicit machine models and into the notion of
{\em trace colorings}: they consider any labeling that assigns
a color to each symbol of each string in the language,
representing an arbitrary step-by-step annotation of the string
\cite{charikar26language}.
The idea in each of these papers is that when the learner is
presented with the annotated version of the string, it
can solve the problem of identification in the limit for some class
of languages.
Charikar et al.~(2026) showed that there is a trace coloring scheme
using a finite set of colors 
for which a learner can solve identification in the limit for
{\em every} countable collection of candidate languages
\cite{charikar26language}.

In terms of the resource requirements of the trace coloring scheme
for annotation, 
one of the key goals is to reduce the number of
colors required, so that the annotation can be as compact as possible
per symbol.
Viewed in these terms, 
the main result of Charikar et al.~(2026) is that there is a trace coloring
using $k+1$ colors, where $k$ is the size of the underlying alphabet.
They also show that for the special case of regular languages, 
they can improve the resource requirements dramatically, via
a trace coloring that uses only two colors.

These results leave open several basic questions.
First, the bound of $k+1$ is strong in one sense, since it
is only linear in the alphabet size, but it raises the natural question
of whether the number of colors in a trace coloring that
supports identification in the limit actually needs
to grow with the alphabet size at all.
Second, a generally
unstated resource constraint for trace coloring is that a color
is assigned to {\em each} symbol of the string --- that is,
the annotation works symbol-by-symbol --- but in principle
we might be able to get away with annotating many fewer of the 
symbols in the string.
The most extreme version of this would be a {\em terminal coloring},
which only provides a single annotation at the end of the entire 
string, and yet still supports identification in the limit.\footnote{As we discuss further in a later section, we can always convert a trace coloring
using $\ell$ colors per symbol into a terminal coloring by taking the
colors applied to each symbol of a string of length $n$ and 
instead writing it as a single $n$-tuple of colors that annotates 
only the end of the string.
This is a useful general conversion, but of course it does not fulfill the goal of using a fixed number
of colors, since such an encoding requires an unbounded number of colors that
grows with the length of the string: there are $\ell^n$
possible $n$-tuples of colors for a string of length $n$.}

\paragraph{The present work: Qualitatively stronger resource bounds.}
In this paper, we show that dramatically stronger resource bounds
are achievable for trace colorings: for any countable collection of languages
(where as before we assume all languages are infinite),
there is a terminal coloring that uses only two colors 
and supports identification in the limit.
This is best possible in both dimensions discussed above:
it uses only two colors, and (as a terminal coloring) it 
provides only one annotation per string.

Essentially, what the result shows is that there is a powerful 
kind of positive, ``non-uniform'' counterpart to Gold's classical
impossibility theorem for identification in the limit:
if we are allowed to annotate each string in each language with
a single one-bit annotation, then identification in the limit
becomes possible for every countable collection of infinite languages.\footnote{It is
useful to note what is possible for language collections where
some of the languages in the collection might be finite sets.
What we can do in this case is to apply the terminal coloring 
to all the infinite languages in the collection, and use a distinct third
color to annotate all the strings in all the finite languages.
This means that as soon as the third color is seen, the algorithm
knows it is working with a finite set, and it can use a standard
algorithm to identify this language in the limit (in brief, by
always guessing that the language consists of exactly the strings
it has seen so far).  In this way, a corollary of our main result
is that when the collection can include languages of finite size,
there is a terminal coloring using only three colors that
supports identification in the limit.}

We develop this result in several steps.
First, we observe that there is a relatively direct construction of
a terminal coloring with two colors if the colors assigned to 
a string $w \in L_i$ can depend not just on the language $L_i$
is belongs to but also the {\em other} languages $L_j$ in the collection.

\begin{theorem}[Collection-dependent 2-coloring, see \Cref{prop:2-coloring-collection-dependent}]
    \label{thm:collection-dependent-intro}
    Let $\mcC=\{L_1,L_2,\dots\}$ be any countable language collection, where every $|L_i|=\infty$. There exists a set of 2-coloring functions $\{c_{L}\}_{L \in \mcC}$ that makes $\mcC$ identifiable in the limit.
\end{theorem}

The above is a useful stepping-stone to our main result, but
it is a much more limited type of a terminal coloring than what we will eventually be seeking, since the colors assigned to the strings in each language $L$ in this first theorem are defined with knowledge of the collection $\mcC$ that $L$ belongs to.
We therefore think of this first result as a {\em collection-dependent} terminal coloring, since the coloring is defined with reference to the collection of languages.
What we will obtain in the next stronger result is a 
{\em collection-independent} way to 
annotate the strings in each language $L$ --- using only knowledge of the single language $L$ --- 
such that when we subsequently assemble a countable collection of languages,
the annotations of each language separately allow for 
identification in the limit.
In other words, this collection-independent coloring is an annotation at the level of the language,
not at the level of the collection of languages.
This stronger definition accords with the intuition from practice, where
annotations are generally providing assistance with the language
in question, not with the contrast to infinite sets of other 
counterfactual languages that are not the language in question.

The challenge we face in going from collection-dependent coloring of the first theorem to
the collection-independent coloring (also with two colors)
turns out to be fundamentally set-theoretic in nature.
We are considering all possible infinite languages --- i.e.,
all infinite subsets of a countable ground set
(which we can take without loss of generality to be the natural numbers) ---
and we want to 2-color the elements of each in such a way that
when we pull together any countable collection of these infinite sets,
it is possible to achieve identification in the limit.
We show that for this purpose, it is sufficient for the colorings
of different subsets to satisfy a set of interdependent requirements 
that we define in later sections; and we use these requirements as part of a
transfinite recursion that iterates through the uncountably many 
infinite languages and builds the 2-coloring of each language with reference
to the 2-colorings that have been inductively defined for all
earlier languages.

\begin{theorem}[Collection-independent 2-coloring, see \Cref{thm:2-coloring-collection-independent}]
    \label{thm:collection-independent-intro}
    Let $\mcC$ be the collection of all infinite subsets of the universe. There exists a set of terminal 2-coloring functions $\{c_{L}\}_{L \in \mcC}$, such that any countable subcollection $\mcC' \subseteq \mcC$ is identifiable in the limit with terminal colors given by $\{c_{L}\}_{L \in \mcC'}$.
\end{theorem}
Note crucially how the construction of the coloring functions above is collection-independent: the terminal coloring used by any given language $L$ is the same, regardless of the countable collection $\mcC'$ it is considered within. This is in contrast to the coloring functions from \Cref{thm:collection-dependent-intro}, where the terminal coloring used by a language $L$ depends on the collection containing it.

\paragraph{The inherently non-constructive nature of bounded terminal colorings.}
This is of course an extremely non-constructive way to define the
terminal coloring, and a basic question that arises is whether
this non-constructive aspect is necessary in every terminal 2-coloring ---
or more generally, in any {\em bounded terminal coloring}, where
we are constrained to use a fixed number of colors.
We show that it is, in a strong sense.
First we need make precise what we mean by ``constructive'' versus
``non-constructive'' when the object being constructed --- a 2-coloring
of every infinite subset of the natural numbers --- is uncountable
in cardinality to begin with.

Fortunately, there is a standard notion of constructiveness in this setting,
which is based on the concept of a {\em Borel map}.
A terminal $k$-coloring is simply a map $J$ that sends infinite
subsets $L$ of the natural numbers to infinite strings over the
finite alphabet $\{1, 2, 3, ..., k\}$ (specifying the color assigned
to each element of the set $L$).
The standard topology on the domain $D$ (all infinite subsets of the naturals), 
and the standard topology on the range $R$ (all infinite strings over
$\{1, 2, 3, ..., k\}$), both define the basic open sets as those sets 
where membership can be determined from the values on a finite prefix.
The Borel sets are then all sets that can be formed from the basic open sets
by iterated taking of unions, intersections, and complements
(in other words, the $\sigma$-algebra generated by the basic open sets).
Just as a map is {\em continuous} if $J^{-1}(Y)$ 
is open for every open set $Y$, 
a map is {\em Borel} if $J^{-1}(Y)$ is Borel for every Borel set $Y$.
Finally, given that a $k$-coloring is just a map from the domain to
the range as we have defined them, we say that a coloring is Borel if the map
specifying it is a Borel map.
Experience shows that
essentially any function that can be defined using a ``simple'' 
or ``closed-form'' iterative construction is Borel, 
and all the colorings in previous work in this area are Borel.

What we show is that for all finite $k$, if a terminal $k$-coloring 
supports identification in the limit, then it cannot be a Borel coloring.
We do this via an infinitary Ramsey-theoretic argument,
using the Galvin-Prikry Theorem \cite{galvinprikry}.

\begin{theorem}[No collection-independent finite Borel coloring, see \Cref{thm:borel-lb-identification}]
    \label{thm:borel-lb-identification-intro}
    Let $\mcC$ be the collection of all infinite subsets of the universe, and let $\{c^J_L\}_{L \in \mcC}$ be terminal coloring functions specified by a Borel map $J$ that uses finitely many colors. There exists a countable subcollection $\mcC' \subseteq \mcC$ such that $\mcC'$ is not identifiable in the limit with terminal colors given by $\{c^J_L\}_{L \in \mcC'}$.
\end{theorem}

This is the deep sense in which our main result is a
highly non-uniform positive counterpart to Gold's Theorem:
there is a way of annotating every string in every possible 
infinite language with a single bit (i.e. one of two colors) such that
when any countable collection is assembled from these annotated languages,
identification in the limit is possible.
But every way of performing such an annotation is necessarily
highly non-constructive, in that the annotation cannot be accomplished
with a Borel map.
This also forms a contrast with the earlier result of Charikar et al.~(2026), 
which requires trace (rather than terminal) colorings and an 
unbounded set of colors (growing with alphabet size): in these
quantitative respects, theirs is a weaker result, 
but it is also constructive in a way that our
terminal 2-coloring cannot be.

\begin{table}[t]
\centering
    \[
    \begin{array}{|c|c|c|}
    \hline 
     & \text{Collection-dependent} & \text{Collection-independent} \\
    \hline
    \text{Constructive (Borel)}
    &
    \begin{array}{c}
    2 \text{ colors }\\
    \text{(\Cref{prop:2-coloring-collection-dependent})}
    \end{array}
    &
    \begin{array}{c}
    {\rm infinite ~ color ~ set} \ \text{(\Cref{thm:borel-map-infinite-palette})}\\
    \text{(provably necessary due to \Cref{thm:borel-lb-identification})}
    \end{array}
    \\
    \hline
    \text{Non-Borel}
    &
    \begin{array}{c}
    2 ~ {\rm colors}\\
    \text{(implied by \Cref{thm:2-coloring-collection-independent})}
    \end{array}
    &
    \begin{array}{c}
    2 ~ {\rm colors}\\
    \text{(\Cref{thm:2-coloring-collection-independent})}
    \end{array}
    \\
    \hline
    \end{array}
    \]
    \caption{Summary of our results. The entries indicate the number of colors required by terminal coloring functions for language identification.}
    \label{table:results}
\end{table}

\section{Overview of Techniques}
\label{sec:overview}

In this section, we give detailed proof sketches for all our main results.

\subsection{Upper Bounds}
\label{sec:upper-bounds-sketch}

\paragraph{\Locall 2-coloring functions.} We start by describing a simple 2-coloring scheme that is collection-dependent, and satisfies the distinguishable coloring condition. Let $\mcC$ be any countable language collection, and suppose all the languages in $\mcC$ are infinite. Our strategy is to enumerate all the pairs in $\mcC$ one by one, and build the terminal coloring functions as we process them. Namely, whenever we encounter a pair $(L,L')$ for which $L \subsetneq L'$, we designate a special ``sentinel'' string $x\in L$ which will be colored red in $L$ (i.e.., $c_L(x)=\Red$) but blue in $L'$ (i.e., $c_{L'}(x) = \Blue$). Furthermore, this sentinel string will be chosen to be \textit{longer} than any sentinel string chosen so far; such a choice is possible since all languages are infinite, and ensures uniqueness of sentinel strings. Finally, for every language, all the strings that never get chosen as sentinels in this (infinite) process are colored blue. We can readily see that the terminal coloring functions thus constructed satisfy the distinguishable coloring condition; indeed, for any $L,L' \in \mcC$ where $L \subsetneq L'$, we will have processed this pair at some finite time in the enumeration above, and hence designated some sentinel $x \in L$ to be colored red in $L$ but blue in $L'$.

\paragraph{\Globall 2-coloring functions.} Notice that the 2-coloring functions constructed above are heavily collection-dependent. A string $x \in L$ may be colored red if it gets chosen as a sentinel within a collection $\mcC$, but may be colored blue within a different collection $\mcC'$. We now describe how it is possible to obtain global 2-coloring functions for languages, such that the distinguishable coloring condition is satisfied with the same coloring functions \textit{irrespective} of the collection in question. Notice that in order for this, it suffices to construct 2-coloring functions for the collection that comprises of \textit{all} possible languages (i.e., all infinite subsets of the universe), since such a construction guarantees that, for every language $L$, and every possible proper superset $L' \supsetneq L$, there is a string $x \in L$ that is colored differently in $L$ and $L'$. However, the collection of all infinite subsets of the universe is \textit{uncountable}, which poses challenges in terms of building the coloring functions by enumerating pairs of languages as in the above. Therefore, we have to appeal to a highly non-constructive tool: transfinite recursion.

For simplicity, and without loss of generality, assume that the universe is the natural numbers $\N$, and let $\mcC$ be the set of all infinite subsets of $\N$. Even if $\mcC$ is uncountable --- i.e., has cardinality $\mfc$, where $\mfc$ is the cardinality of the continuum ---, by the axiom of choice, it is well-ordered, which means, by the well-ordering theorem, that it is order-isomorphic to a unique ordinal having cardinal $\mfc$. So, consider the \textit{smallest} ordinal $\kappa$ having cardinality $\mfc$. For such a $\kappa$, it holds that $|\alpha| < |\kappa|=\mfc$ for every $\alpha < \kappa$; we will crucially exploit this.

Since $\mcC$ and $\kappa$ have the same cardinality, let us put the ordinals $\alpha$ smaller than $\kappa$ in bijection with the languages in $\mcC$, i.e., $\mcC = \{L_\alpha: \alpha < \kappa\}$. Since $|\alpha| < \mfc$ for every $\alpha < \kappa$, 
our strategy will be to associate an \textit{uncountable} family $\mcM_\alpha$ of infinite subsets of $L_\alpha$ with every $L_\alpha$. The family $\mcM_\alpha$ that we choose will have the special property of being \textit{almost-disjoint}: namely, any two members of $\mcM_\alpha$ will only have finite intersection. There are various ways in which such an almost-disjoint, uncountable family of infinite subsets of $L_\alpha$ may be built; we briefly describe one ahead; for now, let us assume the existence of such a family for every $L_\alpha$. We proceed to sketch how we use transfinite recursion to conclude with the desired global 2-coloring functions satisfying the distinguishing coloring condition. 

For every $\alpha < \kappa$, we will carefully choose a member $J_\alpha \in \mcM_\alpha$, with a view to coloring all the elements in $J_\alpha$ as red, and all the elements in $L \setminus J_\alpha$ blue, and in a way that guarantees that the distinguishing coloring condition is not violated within the subcollection  $\{L_\beta\}_{\beta < \alpha}$. At a high level, such a choice is made possible since the number of ordinals $\beta < \alpha$ is at most $|\alpha| < |\kappa|=\mfc$, whereas there are uncountably many available candidates in $\mcM_\alpha$, leaving room for choosing $J_\alpha$. This is where the choice of $\kappa$ as the \textit{initial} ordinal having cardinality $\mfc$, together with the almost-disjointedness property, is significant.

To get a sense of how the almost-disjointedness property helps, consider any $\beta < \alpha$ for which $L_\beta \subsetneq L_\alpha$. Then, in order for the distinguishable coloring condition to be satisfied, we wish for there to be a red-colored string in $L_\beta$ that is colored as blue in $L_\alpha$. That is, we want that $J_\beta \neq J_\alpha \cap L_\beta$. So, any $X \in \mcM_\alpha$ that satisfies $X \cap L_\beta = J_\beta$ is \textit{invalidated} as a candidate for being chosen as $J_\alpha$. But crucially, since $\mcM_\alpha$ is almost-disjoint, there cannot exist distinct $X, Y \in \mcM_\alpha$ that both satisfy this property, since otherwise, $J_\beta \subseteq X \cap Y$, and $|J_\beta|=\infty$. A similar argument holds in the case where $L_\beta \supsetneq L_\alpha$. Thus, every $\beta < \alpha$ invalidates at most a single member of $\mcM_\alpha$. Since $\mcM_\alpha$ is uncountable, we can therefore safely choose a $J_\alpha \in \mcM_\alpha$ to color red.

Finally, we briefly describe one way to construct the required almost-disjoint family. Recall that we wish to construct an uncountable family $\mcM_\alpha$ of infinite subsets of $L_\alpha$ such that every two members of the family have only a finite intersection. We will add a member $L_{\alpha, r} \subseteq L_\alpha$ to $\mcM_\alpha$ for every unique, infinite bit string $r$ (of which there are uncountably many), where the elements of $L_{\alpha, r}$ are determined by all the \textit{finite prefixes} of $r$. Since any two distinct bit strings $r$ and $s$ can agree on only finitely many prefixes, we can associate prefixes with elements within $L_\alpha$, and construct the family $\mcM_\alpha$ such that its members are almost-disjoint.

\subsection{Lower Bounds}
\label{sec:lower-bounds-sketch}

\paragraph{Finite-prefix coloring functions.} We start by showing a lower bound for \textit{finite-prefix} coloring functions, a simple but natural family of \globall coloring functions. A family of finite-prefix coloring functions is defined with respect to a function $f$ that maps finite subsets of $\N$ to the palette, and the terminal coloring function $c^f_L$ for a language $L=\{x_0,x_1,\dots\}$ is defined to be
\begin{align*}
    c^f_L(x_i)=f(\{x_0,x_1,\dots,x_i\}).
\end{align*}
We show that no family of finite-prefix coloring functions with any finite number of colors can globally suffice for language identification. For simplicity, we sketch the argument here for 2 colors; the argument extends naturally to $k > 2$ colors.

First, let us see how finite-prefix 2-coloring functions do not satisfy the distinguishable coloring condition, which is only a sufficient condition for identification; later, we will extend the argument to rule out identification. For this, we must exhibit two infinite subsets $A \subsetneq B$ of the natural numbers satisfying that $f(A_{\le x})=f(B_{\le x})$ for all $x \in A$ (where $A_{\le x}$ comprises of all elements in $A$ that are at most $x$). We attempt to build these sets greedily. That is, we start with $E_0=\{\}, F_0=\{1\}$, and aim to maintain growing set pairs $(E_i, F_i)$ for which it holds that $f(E_{i, \le x})=f(F_{i, \le x})$ for all $x \in E_i$. Given $(E_i, F_i)$, we attempt to find a candidate finite set $S$ of numbers strictly larger than all the numbers in $F_i$, with a view to append $S$ to $F_i$ but only the single element $\max(S)$ to $E_i$, such that the invariant
\begin{align*}
    f(E_i \cup \max(S))=f(F_i \cup S).
\end{align*}
is satisfied. If we can keep finding such sets $S$ and grow our set pairs indefinitely, we can set $A = \cup_{i \ge 0}E_i$ and $B=\cup_{i \ge 0}F_i$. Otherwise, we would halt at some finite set pair $(E_n, F_n)$; however, the halting condition implies that for any finite set $S$ of numbers above all the numbers in $F_n$, it holds that $f(F_n \cup S) \neq f(E_n \cup \max(S))$. In particular, this means that for any number $u > \max(F_n)$, and for any finite set $S$ of numbers larger than $\max(F_n)$ that has $\max(S)=u$, the color $f(F_n \cup S)$ is \textit{pinned down} to be either red or blue (i.e., the color \textit{not equal} to $f(E_n \cup u)$), solely as a function of $u$. We can thus set $A=F_n \cup C$ and $B=F_n \cup D$ for any infinite sets $C, D \subseteq \N_{>\max F_n}$, where $C \subsetneq D$.

We now extend the argument above to construct a countable collection $\mcC'$ for which not only the distinguishable coloring condition, but also identification is precluded. As in the above, we sketch the lower bound for $2$ colors.

By the characterization for identification with color traces in \cite{charikar26language} (see \Cref{thm:characterization-identification-with-terminal-traces}), the lower bound amounts to constructing a collection $\mcC'$ that contains a language $L$, such that for every finite, non-empty subset $T \subseteq L$, there exists a different language $L' \in \mcC'$ for which $T \subseteq L' \subsetneq L$, and furthermore, $f(L'_{\le x})=f(L_{\le x})$ for every $x \in L'$. %
We will construct $\mcC'$ to comprise of a special language $B$, together with languages $A_i$ for $i \in \N$. 

Towards this, recall the situation above in the lower bound for the distinguishable coloring condition, where we were forced to halt in our infinite construction. Here, we had a finite set $G$ which had the property that, for every $u > \max(G)$, and for any finite set $S \subseteq \N_{> \max(G)}$ satisfying $\max(S)=u$, the color $f(G \cup S)$ is \textit{determined} to be either red or blue \textit{solely} as a function of $u$ (a property that we denote later as ``1-coherence''). Namely, if $\N_{> \max(G)}=\{z_0, z_1,\dots\}$, we set $B$ to be $G \cup \N_{> \max(G)}$, and $A_i=G \cup \N_{> \max(G)} \setminus \{z_i\}$. We can readily see that every $A_i \subsetneq B$, and that for every finite, non-empty $T \subseteq B$, there exists $A_i$ satisfying $T \subseteq A_i \subsetneq B$; furthermore, for such an $A_i$, the fact that $f(A_{i, \le x})=f(B_{\le x})$ for every $x \in A_i$ follows simply from the constraint above on the set $G$. So, if we have a set $G$ that is 1-coherent, our job is complete. 

Otherwise, our high-level strategy for constructing $B$, $\{A_i\}_{i \in \N}$ is as follows: we attempt to construct an increasing sequence of finite sets $H_0 \subsetneq H_1 \subsetneq \dots$ in steps. At each step $s$, we also maintain growing finite approximations of $A_0,\dots,A_s$. Denoting by $A_i[s]$ the finite approximation of $A_i$ maintained at step $s$, we maintain the invariant that for every $i$, $A_i[s] \subseteq H_s$, and furthermore, $f(A_i[s]_{\le x})=f(H_{s, \le x})$ for every $x \in A_i[s]$. We initialize $H_0=A_{0}[0]=\{1\}$. Then, at each step $s$, we will attempt to grow exactly one of the finite approximations $A_0[s],\dots,A_s[s]$ by \textit{focusing} on some $i \le s$. When we focus on $A_i[s]$, we attempt to find some finite, non-empty set $S \subseteq \N_{> \max(H_s)}$ such that, denoting $u=\max(S)$, it holds that 
\begin{align*}
    f(A_i[s] \cup \{u\}) = f(H_s \cup S).
\end{align*}
If we are able to find such an $S$, we continue to the next step by setting $H_{s+1}=H_s \cup S$. We then initialize the finite approximation of $A_{s+1}$ as $A_{s+1}[s+1]=H_{s+1}$, set $A_{i}[s+1]=A_i[s] \cup \{u\}$, and set $A_{j}[s+1]=A_j[s]$ for every $j \neq i, s+1$. It is then not too hard to see that if we had inductively maintained the required invariant up until step $s$, our update rule maintains the invariant at step $s+1$. If this construction goes on indefinitely, we will set $B= \cup_{s \in \N} H_s$ and $A_i = \cup_{s \ge i}A_i[s]$ for every $i$. Since we want every $A_i$ to eventually become infinite, we choose the index $i$ that we focus on at step $s$ in such a way that every index $i$ gets focused on at infinitely many steps.

Suppose then that the construction goes on indefinitely. Note that every $A_i \subsetneq B$, since at infinitely many time steps, it is also the case that we do \textit{not} focus on $A_i$. We then claim that for every finite, non-empty $T \subseteq B$, there exists $A_i$ such that $T \subseteq A_i \subsetneq B$, and furthermore, $f(A_{i, \le x})=f(B_{\le x})$ for every $x \in A_i$. Indeed, since $B=\cup_{s \in \N} H_s$, any such $T$ is contained in some $H_s$, and note that $A_s[s]$ is initialized to $H_s$, post which $A_s$ only grows. Lastly, the property that $f(A_{s, \le x})=f(B_{\le x})$ follows by definition of the way that we extend both $A_s$ and $B$. So, in the case that the construction goes on indefinitely, we have the desired lower bound. 

In the other case where the construction gets stuck at some step $s$, by definition of getting stuck, we have the property that, for every $u > \max(H_s)$, and for any finite set $S \subseteq \N_{> \max(H_s)}$ satisfying $\max(S)=u$, it holds that $f(H_s \cup S)$ is necessarily \textit{not} equal to $f(A_{i}[s] \cup \{u\})$ for the particular $i$ that we were focusing on at step $s$. In other words, $f(H_s \cup S)$ is pinned down solely as a function of $\max(S)$, which means that $H_s$ is 1-coherent. We can thus fall back on our argument above to construct the required sets $B, \{A_i\}_{i \in \N}$.

\paragraph{Borel coloring functions.} 

We now transfer the identification lower bound for finite-prefix coloring functions above to the much more expressive class of \textit{Borel} coloring functions over a finite palette. This class includes essentially all natural, constructive notions of coloring functions. For example, the \globall coloring functions given by \cite{charikar26language}, when naturally interpreted as terminal coloring functions, are Borel coloring functions (see \Cref{thm:borel-map-infinite-palette}), albeit over an \textit{infinite} palette. These coloring functions satisfy the distinguishable coloring condition, and hence suffice for identification. Our lower bound thus shows that the use of an infinite palette by these functions is provably necessary. Moreover, this lower bound gives further evidence for the non-constructive nature of the transfinite-recursion based 2-coloring functions from our upper bounds above.

To show a lower bound for Borel coloring functions, we show how we can translate a \globall family of Borel coloring functions that satisfy the characterizing condition of \Cref{thm:characterization-identification-with-terminal-traces} to a \globall family of finite-prefix coloring functions satisfying this condition. This contradicts our lower bound above for finite-prefix coloring functions.

At the core of our reduction is an application of the celebrated Ramsay-theoretic theorem of Galvin and Prikry \cite{galvinprikry}, which shows that all Borel sets are Ramsay. We use this theorem, together with the assumption that the coloring function is Borel, to inductively construct an infinite subsequence of the natural numbers (\Cref{lemma:canonicalization}), over which the application of the Borel coloring function behaves like a finite-prefix coloring function over a finite palette. Concretely, having constructed a finite portion of the required subsequence so far, we recursively thin the remaining tail of the subsequence, such that the application of the Borel map on any subsequence of the tail is color-homogenized. This recursive thinning/color homogenization is made possible by the Galvin-Prikry theorem. 

Once we have constructed this sequence, it is only a matter of identifying the entire domain of the natural numbers with this subsequence, which lets us transfer the tell-tales and coloring discrepancies as required by the characterizing condition of \Cref{thm:characterization-identification-with-terminal-traces} from the Borel coloring function to the derived finite-prefix coloring function. This shows that the derived finite-prefix coloring function suffices for identifying every countable collection --- a contradiction to the lower bound above.

With this overview, we now proceed to formally establishing all the details of our results.

\section{Preliminaries}
\label{sec:preliminaries}

A language in this paper is a subset of $\Sigma^*$ for a finite alphabet set $\Sigma$. Without loss of generality (i.e., by considering any suitable bijection), we will identify $\Sigma^*$ by the set of natural numbers $\N=\{0,1,2,\dots\}$. A language $L$ is then simply a subset of $\N$, and a language collection is a (multi)set of subsets of $\N$.

\textit{Ordinal numbers} will feature abundantly in our discussion; following convention, we will generally denote these by Greek alphabets. We will use the von Neumann definition of ordinal numbers: a set $\alpha$ is an ordinal, if and only if every element of $\alpha$ is a subset of $\alpha$, and furthermore, $\alpha$ is (strictly) well-ordered by set-membership $(\in)$. Namely, every non-empty subset $\beta \subseteq \alpha$ has a unique least element $a$ such that $a < b$ (i.e., $a \in b$) for every $b \in \beta \setminus \{a\}$. Observe that any ordinal $\alpha$ may be equivalently defined as $\alpha = \{\gamma: \gamma < \alpha\}$.

For example, finite ordinals corresponding to the natural numbers are constructed as $0=\{\}$, $1=\{0\}$, $2=\{0,1\}$, and so on. The first infinite ordinal $\omega$ consists of all the finite ordinals, and is hence equivalent to the set $\N=\{0,1,2,\dots\}$. Thereafter, $\omega+1 = \{0,1,2,\dots,\omega\}$, and so on. The central notion that an ordinal captures is its unique \textit{order type} --- informally, this is the \textit{shape} of the order. For example, $2$ has the shape $0 < 1$, $\omega$ has the shape $0 < 1 < \dots$ with no largest element, $\omega + 1$ has the shape $0 < 1 < \dots < \omega$ with $\omega$ being the largest element --- each of these shapes are different.

The cardinality of a set $S$, denoted $|S|$, is its size. Two sets $S$ and $S'$ have the same cardinality if there exists a bijection between them. For example, observe that $|\omega|=|\omega+1|$. We will denote the cardinality of the set of real numbers $\R$ by $\mfc$.

In this paper, we will assume the \textit{Axiom of Choice}. The axiom of choice is equivalent to the \textit{Well-ordering Theorem}, which asserts that every set can be well-ordered. It is known that every well-ordered set $(S, <)$ is order-isomorphic to a \textit{unique} ordinal number $\alpha$, meaning that there exists a bijection $f$ between $S$ and $\alpha$, such that for any $a,b \in S$, $a < b \iff f(a) \in f(b)$.

A terminal $k$-coloring function $c$ maps $\N$ to $\palette$, where $\palette$ is a palette of $k$ colors. The following characterization of identification in the limit with terminal color traces essentially follows from Theorem 1 in \cite{charikar26language} --- while the result in \cite{charikar26language} is stated for \textit{trace} coloring functions, the argument for terminal coloring functions is identical.

\begin{restatable}[Characterization of Identification with Terminal Color Traces, essentially \cite{charikar26language}]{theorem}{characterization}
    \label{thm:characterization-identification-with-terminal-traces}
    Let $\mcC$ be a countable language collection. Then, $\mcC$ is identifiable in the limit with a terminal color trace given by the terminal coloring functions $\{c_L\}_{L \in \mcC}$ if and only if for every language $L \in \mcC$, there exists a finite ``tell-tale'' subset $T_L \subseteq L$, such that for every language $L' \in \mcC$ that is a proper subset of $L$, either (1) $L'$ does not contain $T_L$, or (2) there exists $x \in L'$ such that $c_{L'}(x) \neq c_{L}(x)$.
\end{restatable}
The condition (2) in the characterization above, which is a sufficient condition for identification with colors traces, is denoted in \cite{charikar26language} as the \textit{distinguishable coloring condition}.

\paragraph{Distinguishable coloring condition.} A set of terminal $k$-coloring functions $\{c_{L}\}_{L \in \mcC}$ for a collection $\mcC=\{L_1,L_2,\dots\}$ satisfies the ``distinguishable coloring condition'' if for every $L, L' \in \mcC$ satisfying $L' \subsetneq L$, there exists $x \in L'$ for which
        $c_{L'}(x) \neq c_{L}(x)$.

\section{Collection-independent 2-coloring}
\label{sec:universal-2-coloring}

The simple proposition ahead shows, for every countable language collection, the existence of collection-dependent 2-coloring functions that satisfy the distinguishable coloring condition.

\begin{proposition}
    \label{prop:2-coloring-collection-dependent}
    Let $\mcC=\{L_1,L_2,\dots\}$ be any countable language collection, where every $|L_i|=\infty$. There exists a set of terminal 2-coloring functions $\{c_{L}\}_{L \in \mcC}$ that satisfies the distinguishable coloring condition, and hence makes $\mcC$ identifiable in the limit.
\end{proposition}
\begin{proof}
    Let us enumerate all the pairs $(L_i, L_j)$ where $i < j$ as $(L_{i_1}, L_{j_1}), (L_{i_2}, L_{j_2}), \dots$. We will initialize a size variable $s_0=1$, and process these pairs in order. At step $n$ (where $n=1,2,\dots$), we will first check if either $L_{i_n} \subsetneq L_{j_n}$ or $L_{j_n} \subsetneq L_{i_n}$; if so, assume without loss of generality that $L_{i_n} \subsetneq L_{j_n}$. We will now find a string $x_{n} \in L_{i_n}$ that has size at least $s_{n-1}$, and mark this string as ``special''. Such a special string necessarily exists in $L_{i_n}$ since we assume every language is infinite, and the alphabet is finite. Then, we set $c_{L_{i_n}}(x_n)=\Red$, $c_{L_{j_n}}(x_n)=\Blue$, and update $s_n=|x_n|+1$. Otherwise, if neither $L_{i_n} \subsetneq L_{j_n}$ nor $L_{j_n} \subsetneq L_{i_n}$, we set $s_n=s_{n-1}$, and proceed to the next pair.

    For every language $L \in \mcC$, and for every $x \in \Sigma^*$ that was never chosen to be special at any time in the process above, we set $c_L(x)=\Blue$. This completes the specification of the terminal 2-coloring functions $\{c_L\}_{L \in \mcC}$. Observe that since we always strictly increase the size of special strings, a string $x_n$ chosen to be special at step $n$ is never chosen to be special at any other step. Thus, at no point do we ever overwrite colors.

    It remains to argue that the distinguishable coloring condition is satisfied. Fix any $L, L' \in \mcC$ satisfying $L \subsetneq L'$. These languages must exist as some pair $(L_{i_n}, L_{j_n})$ in our enumeration above; assume without loss of generality that $L_{i_n}=L, L_{j_n}=L'$. Then, the construction above would have found a special $x_n \in L_{i_n}$ and set $c_{L_{i_n}}(x_n)=\Red$, $c_{L_{j_n}}(x_n)=\Blue$. Because we never overwrite colors, these are the final colors assigned to $x_n$ by $c_{L_{i_n}}$ and $c_{L_{j_n}}$. Thus, the distinguishable coloring condition is satisfied.
\end{proof}

\begin{remark}[Finite Languages I]
    \label{remark:finite-languages-1}
    If the collection $\mcC$ also comprises of finite languages, the argument above is easily modified to work with 3 colors. Let $\mcC_1$ comprise of all the finite languages in $\mcC$, and let $\mcC_2=\mcC \setminus \mcC_1$ comprise of all the infinite languages. Since $\mcC$ is countable, both $\mcC_1$ and $\mcC_2$ are countable. We construct terminal coloring functions $\{c_L\}_{L \in \mcC_2}$ that use the colors red and blue, and satisfy the distinguishable coloring condition within $\mcC_2$, as in the above. For the finite languages in $\mcC_1$, we use terminal coloring functions $\{c_L\}_{L \in \mcC_1}$ that are constant functions mapping to the color $\Green$. In this case, if the input comprises of either red or blue traces, the learning algorithm knows that the target language is in $\mcC_2$, and hence operates only within $\mcC_2$; since the traces satisfy the distinguishable coloring condition within $\mcC_2$, identification is successful. Otherwise, if the terminal traces are green, the target language is determined to be a finite language from $\mcC_1$. Since all languages in $\mcC_1$ are finite, the collection is identifiable \cite{gold1967language}. In particular, the input will eventually comprise of all the finitely many strings from the target language, and hence the algorithm can simply output the language in $\mcC_1$ that is \textit{equal} to the observed set of strings.
\end{remark}

The terminal 2-coloring function $c_L$ for a language $L$ above depends on the specific collection $\mcC$ that the language $L$ is considered within. One can ask: does there exist a \textit{universal} terminal 2-coloring function $c_L$ that can be used for language $L$ uniformly across all countable collections $\mcC$ that contain $L$? This amounts to asking: does there exist a set of terminal 2-coloring functions $\{c_{L}\}_{L \in \mcC}$ that satisfies the distinguishable coloring condition, where $\mcC$ contains all infinite subsets of $\Sigma^*$?

\begin{theorem}
    \label{thm:2-coloring-collection-independent}
    Let $\mcC$ be the collection of all infinite subsets of $\N$. There exists a set of terminal 2-coloring functions $\{c_{L}\}_{L \in \mcC}$ that satisfies the distinguishable coloring condition.
\end{theorem}
\begin{proof}
    Since the set of all infinite subsets of $\N$ is uncountable, note that $|\mcC| = \mfc$, where $\mfc$ is the cardinality of the continuum. By the well-ordering theorem, $\mcC$ can be well-ordered. This implies that $\mcC$ is order-isomorphic to a unique ordinal having cardinality $\mfc$. Let $\kappa$ be the \textit{initial ordinal} having this cardinality. That is, $|\kappa| = \mfc$, and furthermore $|\alpha| < |\kappa|$ for every ordinal $\alpha < \kappa$. Since $\mcC$ and $\kappa$ have the same cardinality, choose a bijection $e:\kappa \to \mcC$, where we define $L_\alpha := e(\alpha)$ for every $\alpha < \kappa$. We thus have that
    \begin{align}
        \label{eqn:def-C}
        \mcC = \{L_\alpha: \alpha < \kappa\}.
    \end{align}

    We will now associate every language $L_{\alpha}$ with a family of infinite subsets of it, where this family will have size $\mfc$. Furthermore, the family will be \textit{almost-disjoint}, meaning that any two members within it have a finite intersection. These two properties will ensure the existence of at least a single infinite subset within the family, such that if the terminal coloring function $c_{L_\alpha}$ colors all the elements of this subset red, and all the rest of the elements of $L_\alpha$ blue, the distinguishable coloring condition is satisfied.
    
    Towards building such families for every $L_\alpha$, let us begin with a building block. Let $\mcZ$ denote the set of all infinite bit strings; note that $|\mcZ|=\mfc$. %
    Additionally, fix any bijection $b$ that maps the set of all finite bit strings to $\N$. Then, for any infinite bit string $r \in \mcZ$, define $M_r$ to comprise of all $b(r_{\le i})$ values for $i \ge 0$, where $r_{\le i}$ is the finite prefix $r_0,\dots,r_i$ of $r$. Concretely,
    \begin{align}
        \label{eqn:def-M_r}
        M_r := \{b(r_{\le i}): i \ge 0\}.
    \end{align}
    Note that $|M_r|=|\omega|$. We now claim that $|M_r \cap M_s| < \infty$ for any $r \neq s$. To see this, fix some $r \neq s$, and let $i \ge 0$ be the leftmost index at which $r_i \neq s_i$. Then, observe that for any $j \ge i$, $r_{\le j} \neq s_{\le j}$, and hence $ r_{\le j} \notin M_s$, $s_{\le j} \notin M_r$. Thus, $|M_r \cap M_s|=|\{b(r_{\le j}):j < i\}| < \infty$ as desired.

    Now fix any $\alpha < \kappa$, and let $L_\alpha=\{\ell^\alpha_0 < \ell^\alpha_1 < \dots\}$ under the usual ordering of the natural numbers. Since $|L_{\alpha}|=|\omega|$, consider the bijection $\pi_\alpha$ mapping $\omega$ to $L_{\alpha}$, where
    \begin{align*}
        \pi_\alpha(n) = \ell^\alpha_n, \qquad n \ge 0.
    \end{align*}
    Then, for each $r \in \mcZ$, define 
    \begin{align*}
        L_{\alpha, r} := \pi_\alpha(M_r) = \{\pi_\alpha(n): n\in M_r\} = \{\ell^\alpha_n: n \in M_r \},
    \end{align*}
    where $M_r$ is as defined above. Observe that $|L_{\alpha, r}|=|M_r|=|\omega|$ and $L_{\alpha, r} \subseteq L_\alpha$; furthermore, for any $r \neq s$,
    \begin{align*}
        |L_{\alpha, r} \cap L_{\alpha, s}| = |\{\ell^\alpha_n: n \in M_r \cap M_s \}| < \infty,
    \end{align*}
    which also implies that every $L_{\alpha,r}$ is distinct. So, we define the family $\mcM_\alpha$ as
    \begin{align*}
        \mcM_\alpha = \{L_{\alpha, r}: r \in \mcZ\}.
    \end{align*}
    Since $|\mcZ|=\mfc$, we have that $|\mcM_\alpha|=\mfc$.

    We will now use transfinite recursion to define a sequence $\{J_{\alpha}\}_{\alpha < \kappa}$, where every $J_\alpha \in \mcM_{\alpha}$, $J_\alpha \neq \emptyset$, and for every $\beta < \alpha$, if $L_\beta \subsetneq L_\alpha$, then $J_\alpha \cap L_\beta \neq J_\beta$, and if $L_\beta \supsetneq L_\alpha$, then $J_\beta \cap L_\alpha \neq J_\alpha$.

    As the base case, define $J_0$ to be an arbitrary element of $\mcM_0$. Now fix any ordinal $\alpha < \kappa$.
    and assume that $J_\beta$ satisfying the required condition for every $\beta < \alpha$ has been defined. 
    We will now specify how $J_\alpha$ is defined. 

    For every $\beta < \alpha$, define a subfamily $\mcF_{\alpha, \beta} \subseteq \mcM_\alpha$ of forbidden candidates as follows.

    \paragraph{Case 1: $L_\beta \subsetneq L_\alpha$.} In this case, define
    \begin{align*}
        \mcF_{\alpha, \beta} = \{X : X \in \mcM_\alpha, X \cap L_\beta = J_\beta\}.
    \end{align*}
    We claim that $|\mcF_{\alpha, \beta}| \le 1$. To see this, suppose $X, Y \in \mcF_{\alpha, \beta}$ and $X \neq Y$. Then, $X \cap L_\beta = J_\beta$, and also, $Y \cap L_\beta = J_\beta$. This implies that  $J_\beta \subseteq X \cap Y$, meaning that $|X \cap Y|=\infty$, which contradicts that the members of $\mcM_\alpha$ are almost-disjoint. 

    \paragraph{Case 2: $L_\beta \supsetneq L_\alpha$.} In this case, define
    \begin{align*}
        \mcF_{\alpha, \beta} = \{X : X \in \mcM_\alpha, J_\beta \cap L_\alpha = X\}.
    \end{align*}
    Note again that $|\mcF_{\alpha, \beta}| \le 1$, since $J_\beta \cap L_\alpha$ is some fixed set, and depending on whether it exists in $\mcM_\alpha$, $\mcF_{\alpha, \beta}$ is either empty or singleton.

    \paragraph{Case 3: Neither $L_\beta \subsetneq L_\alpha$ nor $L_\beta \supsetneq L_\alpha$.} In this case, define $\mcF_{\alpha, \beta} = \emptyset$.

    In all cases, it holds that $|\mcF_{\alpha, \beta}| \le 1$. Now define
    \begin{align*}
        \mcF_\alpha := \bigcup_{\beta < \alpha} \mcF_{\alpha, \beta}.
    \end{align*}
    We claim that $|\mcF_\alpha| < |\mcM_\alpha|$. Indeed, recall that $|\mcM_\alpha|=\mfc$; furthermore, $|\alpha| < |\kappa| = \mfc$, since $\kappa$ was chosen to be the initial ordinal having cardinality equal to $\mfc$. Since every $|\mcF_{\alpha, \beta}| \le 1$, we get that $|\mcF_\alpha| \le |\alpha| < \mfc = |\mcM_\alpha|$. But since $\mcF_\alpha \subseteq \mcM_\alpha$, this implies that $\mcM_\alpha \setminus \mcF_\alpha \neq \emptyset$. So, define $J_\alpha$ to be an arbitrary element from $\mcM_\alpha \setminus \mcF_\alpha$. Then, $J_\alpha \in \mcM_\alpha, J_\alpha \neq \emptyset$. Furthermore, for every $\beta < \alpha$, if $L_\beta \subsetneq L_\alpha$, then we fall under Case 1 above, wherein any $X \in \mcM_\alpha$ which satisfies $X \cap L_\beta=J_\beta$ would have been excluded as a candidate for $J_\alpha$, ensuring that $J_\alpha \cap L_\beta \neq J_\beta$. If $L_\beta \supsetneq L_\alpha$, we fall under Case 2 above, wherein any $X \in \mcM_\alpha$ that satisfies $J_\beta \cap L_\alpha = X$ would have been excluded as a candidate for $J_\alpha$, ensuring that $J_\beta \cap L_\alpha \neq J_\alpha$. Thus, $J_\alpha$ satisfies all the required conditions. 

    By transfinite recursion, the sequence $\{J_\alpha\}_{\alpha < \kappa}$ is thus well-defined.
    
    Now, let $\alpha, \beta$ be any distinct ordinals. Since the ordinals are totally ordered by membership, either $\alpha < \beta$ or $\beta < \alpha$. Thus, our construction above ensures that for any distinct ordinals $\alpha, \beta$, if $L_{\alpha} \subsetneq L_\beta$, it holds that $J_\beta \cap L_\alpha \neq J_\alpha$.

    We now specify the terminal coloring functions. Recall from \eqref{eqn:def-C} above that the languages $L_\alpha$ for $\alpha < \kappa$ uniquely range over all of $\mcC$. Then, for every $\alpha < \kappa$, define 
    \begin{align}
        \label{eqn:def-universal-2-coloring}
        c_{L_\alpha}(x) = \begin{cases}
            \Red & x \in J_\alpha, \\
            \Blue & \text{otherwise.}
        \end{cases}
    \end{align}
    It remains to argue that these terminal coloring functions satisfy the distinguishable coloring condition. Fix any ordinals $\alpha, \beta$ where $L_\alpha \subsetneq L_\beta$. We must show that there exists $x \in L_\alpha$ for which $c_{L_\alpha}(x) \neq c_{L_\beta}(x)$. Indeed, by our reasoning above, it holds that $J_\beta \cap L_\alpha \neq J_\alpha$. This implies either the existence of some $x \in J_\alpha$ which does not belong to $J_\beta \cap L_\alpha$, or the existence of some $x \in J_\beta \cap L_\alpha$ which does not belong to $J_\alpha$. In the former case, since $J_\alpha \subseteq L_\alpha$, such an $x$ satisfies that $x \notin J_\beta$, meaning that $c_{L_\alpha}(x)=\Red$, $c_{L_\beta}(x)=\Blue$. In the latter case too, we obtain an $x \in L_\alpha$ satisfying $x \in J_\beta$ and $x \notin J_\alpha$, meaning that $c_{L_\alpha}(x)=\Blue$, $c_{L_\beta}(x)=\Red$. Thus, in either case, we ensure the existence of an $x \in L_\alpha$ that satisfies $c_{L_\alpha}(x) \neq c_{L_\beta}(x)$, and the proof is complete.
\end{proof}

\begin{remark}[Finite Languages II]
    \label{remark:finite-languages-2}
    As in \Cref{remark:finite-languages-1}, we can extend $\mcC$ to also include all finite subsets of $\N$, and construct collection-independent terminal coloring functions using 3 colors. For this, we can have $\mcC_1$ comprise of all the finite languages in $\mcC$, and $\mcC_2=\mcC \setminus \mcC_1$ comprise of all the infinite languages. $\mcC_2$ is uncountable, whereas $\mcC_1$ is countable, since the collection of all finite finite subsets of $\N$ is countable. For $\mcC_1$, we use terminal coloring functions $\{c_L\}_{L \in \mcC_1}$ that are constant functions mapping to the color $\Green$, and for $\mcC_2$, we use the terminal coloring functions $\{c_L\}_{L \in \mcC_2}$ as given by the transfinite recursion argument above.
\end{remark}

\section{No collection-independent Borel finite-colorings}
\label{sec:lower-bounds}

The terminal 2-coloring functions from the section above arise from a transfinite recursion argument. One can ask: is there a more constructive procedure for collection-independent 2-coloring functions? As a simpler target, how about a constructive procedure for $k$-coloring functions, for some finite $k$? Recall that the terminal coloring functions arising from the trace coloring functions in \cite{charikar26language} are collection-independent and constructive, but use infinitely many colors.

\subsection{Finite-prefix coloring functions}
\label{sec:finite-prefix-coloring-functions}

As a starting point for ``simple'', constructive but collection-independent terminal coloring functions, we begin by considering the family of \textit{finite-prefix} coloring functions. Recall that we assume, without loss of generality, that a language $L=\{x_0,x_1,\dots,\}$ is an infinite subset of $\N$, listed in the natural increasing order of $\N$. A finite-prefix coloring function $c_L$ assigns a color to $x_i$ solely as a function of $\{x_0,\dots,x_i\}$.

\begin{definition}[Finite-prefix coloring functions]
    Let $\mcC$ be the collection of all infinite subsets of $\N$, where every language $L \in \mcC$ is enumerated in its natural increasing order.
    A family of finite-prefix coloring functions $\{c^f_L\}_{L \in \mcC}$ is specified by a function $f$ that maps finite (ordered) subsets of $\N$ to the palette $\palette$, and which satisfies that, for any language $L = \{x_0,x_1,\dots\}$ in its increasing enumeration,
    \begin{align*}
        c^f_L(x_i) = f(\{x_0,x_1,\dots,x_i\}).
    \end{align*}
\end{definition}

We will show that finite-prefix coloring functions, that are allowed to use an arbitrarily large but finite palette size, do not suffice globally for language identification. Namely, we can construct a collection, where it is impossible to identify the target language even with terminal colors given by the coloring functions.

As a warm-up, we will first show that, for any given family of finite-prefix coloring functions that use a palette of size $k$, there exists a countable collection that does not satisfy the distinguishable coloring condition (recall, the distinguishable coloring condition is only a sufficient condition for language identification). We will later strengthen this argument to show that there exists a countable collection for which identification is impossible.

\begin{theorem}
    \label{thm:finite-prefix-lb-distinguishable-coloring-condition}
    Let $\mcC$ be the collection of all infinite subsets of $\N$, and let $\{c^f_L\}_{L \in \mcC}$ be any family of finite-prefix coloring functions specified by the function $f$, and mapping to a finite palette $\palette$. There exists a countable collection $\mcC'$ such that $\{c^f_L\}_{L \in \mcC'}$ does not satisfy the distinguishable coloring condition.
\end{theorem}

In order to prove the theorem above, we will define the notion of a ``$t$-coherent set''.  For a set $X \subseteq \N$, let $X_{> a}$ (respectively $X_{< a}$) denote all the natural numbers in $X$ that are greater (respectively lesser) than $a$. For any finite $G \subseteq \N$, let $\max(G)$ denote the largest number in $G$, with the convention that $\max(\emptyset)=-1$.

\begin{definition}[$t$-coherent set]
    Let $G \subseteq \N$ be a finite set with $g=\max(G)$. For any $t \ge 1$, we say that $G$ is $t$-coherent with respect to the function $f$ if for every $u > g$, there exists a subset $\palette_u \subseteq \palette$ of the palette of size at most $t$, such that for every finite $S \subseteq \N_{> g}$ having $\max(S)=u$, it holds that $f(G \cup S) \in \palette_u$.
\end{definition}

In words, if we append finitely many elements to $G$, then the evaluation of $f$ on the resulting set belongs to a set of $t$ colors, where this set depends \textit{only} on the largest element appended to $G$. We will use the convention that no set is $0$-coherent with respect to $f$.

Note that in order to prove \Cref{thm:finite-prefix-lb-distinguishable-coloring-condition}, it suffices to construct two infinite sets $A, B$ satisfying $A \subsetneq B$, such that for every $x \in A$, $c^f_A(x)=c^f_B(x)$; by the definition of finite-prefix coloring functions, this follows if $f(A_{\le x})=f(B_{\le x})$ for every $x \in A$.

We will now prove the following lemma by induction on the palette size.

\begin{lemma}
    \label{lemma:induction-coherence}
    Let $G \subseteq \N$ be a finite set, %
    and suppose that $G$ is $t$-coherent with respect to the function $f$, for $t \ge 1$. Then, at least one of the following is true: (1) there exist infinite subsets $A$ and $B$ of $\N$ satisfying $A \subsetneq B$ such that $f(A_{\le x})=f(B_{\le x})$ for every $x \in A$, (2) there exists finite $G' \supseteq G$ such that $G'$ is $(t-1)$-coherent with respect to $f$.
\end{lemma}

\Cref{thm:finite-prefix-lb-distinguishable-coloring-condition} follows from the lemma above, since for any finite-prefix coloring functions specified by a function $f$, and that use a finite palette $\palette$ of size $k \ge 1$, it holds that \textit{every} finite subset of $\N$ is $k$-coherent with respect to $f$, together with our convention that no finite set is $0$-coherent with respect to $f$.

\begin{proof}[Proof of \Cref{lemma:induction-coherence}]
    We will prove the lemma by induction on $t$. For the base case, let $t=1$, and suppose that $G$ is $1$-coherent with respect to $f$. We will construct infinite sets $A, B$ satisfying $A \subsetneq B$ such that $f(A_{\le x}) = f(B_{\le x})$ for every $x \in A$. Let $g=\max(G)$, and let $C$ and $D$ be any infinite subsets of $\N_{> g}$ satisfying $C \subsetneq D$. We will set $A = G \cup C$ and $B = G \cup D$. It remains to argue that $f(A_{\le x}) = f(B_{\le x})$ for every $x \in A$. If $x \in G$, then $A_{\le x}=B_{\le x}=G_{\le x}$, and the desired equality holds. Otherwise, $x \in C$, meaning also that $x > g$. In this case, defining $S_1 := C_{\le x}$ and $S_2 := D_{\le x}$, we have that $A_{\le x}= G \cup S_1$ and $B_{\le x}=G \cup S_2$, and furthermore, $\max(S_1)=\max(S_2)=x$. Then, by definition of $G$ being $1$-coherent, there exists a \textit{single} color $c_x$ such that $f(G \cup S_1)=f(G \cup S_2)=c_x$; this immediately gives that $f(A_{\le x})=f(B_{\le x})$ as desired.

    Suppose now that $t > 1$, and the lemma holds for all $t' < t$. Let $G$ be the given set that is $t$-coherent with respect to $f$, and let $g=\max(G)$. Set $E_0=G$, and $F_0=G \cup \{g+1\}$. Starting with $(E_0,F_0)$, we will attempt to construct an infinite chain $(E_0,F_0) \to (E_1,F_1) \to \dots$, where every $E_i,F_i$ is finite, $E_i \subsetneq E_{i+1}$, $F_i \subsetneq F_{i+1}$, $E_i \subsetneq F_i$, and furthermore, for any $x \in E_i$, it holds that $f(E_{i, \le x})=f(F_{i, \le x})$.

    Observe first that $(E_0, F_0)$ already satisfies that $E_0 \subsetneq F_0$, and also that $f(E_{0, \le x})=f(F_{0, \le x})$ for every $x \in E_0$. Now suppose that $(E_i, F_i)$ has been built, and let $h = \max(F_i)$. Let $S \subseteq \N_{> h}$ be any non-empty, finite set, for which, denoting $\max(S)=u$, it holds that
    \begin{align}
        \label{eqn:extension-step}
        f(E_i \cup \{u\}) = f(F_i \cup S).
    \end{align}
    If such a set $S$ exists, we set $E_{i+1}=E_i \cup \{u\}$, and $F_{i+1}=F_i \cup S$. Then, consider any $x \in E_{i+1}$; if $x \in E_i$, $E_{i+1, \le x}=E_{i,\le x}$ and $F_{i+1, \le x}=F_{i, \le x}$, and $f(E_{i+1, \le x})=f(F_{i+1, \le x})$ simply because $f(E_{i, \le x})=f(F_{i, \le x})$. Otherwise, $x =u$; in this case, by construction $E_{i+1, \le x}=E_i \cup \{u\}$ and $F_{i+1, \le x}=F_i \cup S$, and hence, $f(E_{i+1, \le x})=f(F_{i+1, \le x})$ by \eqref{eqn:extension-step}. The other requirements, namely $E_{i} \subsetneq E_{i+1}, F_i \subsetneq F_{i+1}$ and $E_{i+1} \subsetneq F_{i+1}$, can also be readily verified.

    Now suppose that we are able to construct an infinite chain of $(E_i, F_i)$. In this case, we set $A = \cup_{i \ge 0} E_i$, $B = \cup_{i \ge 0} F_i$. Note that $A$ and $B$ thus constructed are both infinite sets, and satisfy that $A \subsetneq B$. This is because the initial element $g+1 \in F_0 \setminus E_0$ belongs to $B$, and is never later added to any $E_i$, since all later additions to $E_i$ are above $\max(F_0)=g+1$. Furthermore, any $x \in A$ belongs to some $E_i$, and $A_{\le x}=E_{i, \le x}$, $B_{\le x}=F_{i, \le x}$, which also means that $f(A_{\le x})=f(B_{\le x})$ by construction. Thus, in this case, the inductive step holds because we satisfy (1) in the lemma statement.

    Otherwise, the construction of the chain stops at $(E_n, F_n)$ for some finite $n$. Let $h=\max(F_n)$; by definition of stopping, it must be the case that for every finite $S \subseteq \N_{> h}$, where we denote $\max(S)=u$, it holds that 
    \begin{align}
        f(F_n \cup S) \neq f(E_n \cup \{u\}).
    \end{align}
    But now observe that $F_n \cup S = G \cup T_1$ and $E_{n} \cup \{u\}=G \cup T_2$ for some finite sets $T_1,T_2 \subseteq \N_{> \max(G)}$, which satisfy that $\max(T_1)=\max(T_2)=u$. Since $G$ is $t$-coherent with respect to $f$, this means that both $f(F_n \cup S)$ and $f(E_n \cup \{u\})$, belong to some palette $\palette_u$, where $|\palette_u| \le t$. Since $f(F_n \cup S) \neq f(E_n \cup \{u\})$, we conclude that, for every $u > h=\max(F_n)$, and for every finite $S \subseteq \N_{> h}$ that satisfies that $\max(S)=u$, $f(F_n \cup S) \in \palette'_u$, where $|\palette'_u| \le t-1$. Thus, setting $G'=F_n \supseteq G$ renders that $G'$ is $(t-1)$-coherent with respect to $f$; the induction is therefore complete.
\end{proof}

Building up on the lower bound proof above, we will now show that for any given family of finite-prefix coloring functions that use a palette of size $k$, there exists a countable collection for which identification is impossible.

\begin{theorem}
    \label{thm:finite-prefix-lb-identification}
    Let $\mcC$ be the collection of all infinite subsets of $\N$, and let $\{c^f_L\}_{L \in \mcC}$ be any family of finite-prefix coloring functions specified by the function $f$, and mapping to a finite palette $\palette$. There exists a countable collection $\mcC'$ such that $\mcC'$ is not identifiable in the limit with terminal color traces given by $\{c^f_L\}_{L \in \mcC'}$.
\end{theorem}

By the characterizing condition in \Cref{thm:characterization-identification-with-terminal-traces}, it suffices to construct a countable collection $\mcC'$ which contains a language $L$, such that for any finite, non-empty subset $T \subseteq L$, there exists a different language $L' \in \mcC'$ such that $T \subseteq L' \subsetneq L$, and furthermore, $f(L'_{\le x})=f(L_{\le x})$ for every $x \in L'$. Using our convention that \textit{every} finite subset of $\N$ is $k$-coherent with respect to $f$ when $f$ uses a palette of size $k$, it suffices to prove the following lemma to establish \Cref{thm:finite-prefix-lb-identification}.

\begin{lemma}
    \label{lemma:induction-coherence-2}
    Let $G \subseteq \N$ be a finite set, %
    and suppose that $G$ is $t$-coherent with respect to the function $f$, for $t \ge 1$. Then, at least one of the following is true: (1) there exist infinite sets $B$, $\{A_i\}_{i \in \N}$ that are subsets of $\N$, such that for every finite, non-empty $T \subseteq B$, there exists $i \in \N$ such that $T \subseteq A_i \subsetneq B$, and $f(A_{i,\le x})=f(B_{\le x})$ for every $x \in A_i$, (2) there exists finite $G' \supseteq G$ such that $G'$ is $(t-1)$-coherent with respect to $f$.
\end{lemma}

\begin{proof}
    We will prove the lemma by induction on $t$. For the base case, let $t=1$, and suppose that $G$ is $1$-coherent with respect to $f$. We will construct infinite sets $B$, $\{A_i\}_{i\in \N}$ such that for every finite, non-empty $T \subseteq B$, there exists $i \in \N$ such that $T \subseteq A_i \subsetneq B$ and $f(A_{i,\le x})=f(B_{\le x})$ for every $x \in A_i$. Let $g=\max(G)$, and let $D=\{z_0 < z_1 < \dots\}$ be any infinite subset of $\N_{> g}$. We will set $B = G \cup D$, and $A_i = G \cup (D \setminus \{z_i\})$. Fix any finite, non-empty $T \subseteq B$,  and let $s = \max(T)$. Consider any $A_i$, where $z_i > s$; it is clear that $T \subseteq A_i \subsetneq B$. It remains to argue that $f(A_{i, \le x}) = f(B_{\le x})$ for every $x \in A_i$. If $x \in G$, then $A_{i, \le x}=B_{\le x}=G_{\le x}$, and the desired equality holds. Otherwise, $x \in D \setminus \{z_i\}$, meaning also that $x > g$. In this case, defining $S_1 := (D \setminus \{z_i\})_{\le x}$ and $S_2 := D_{\le x}$, we have that $A_{i, \le x}= G \cup S_1$ and $B_{\le x}=G \cup S_2$, and furthermore, $\max(S_1)=\max(S_2)=x$. Then, by definition of $G$ being $1$-coherent, there exists a \textit{single} color $c_x$ such that $f(G \cup S_1)=f(G \cup S_2)=c_x$; this immediately gives that $f(A_{i, \le x})=f(B_{\le x})$ as desired.

    Suppose now that $t > 1$, and the lemma holds for all $t' < t$. Let $G$ be the given set that is $t$-coherent with respect to $f$, and let $g=\max(G)$. We will attempt to construct an infinite chain of finite sets $H_0 \subsetneq H_1 \subsetneq H_2 \cdots$, starting with $H_0=G$. At step $s$, we will also maintain a family of finite sets $A_i[s]$ for every $i \le s$. These sets will satisfy that: 
    \begin{align}
        \label{eqn:stage-building-property}
        A_i[s] \subseteq H_s, \text{ and } f(A_i[s]_{\le x}) = f(H_{s, \le x}) \text{ for every $x \in A_i[s]$, and for every $i \le s$.}
    \end{align}
    The sets $A_i[s]$ for $s \ge i$ are intended to be increasing finite approximations of the set $A_i$. We will incrementally build up the sets $H_s$ and $A_i[s]$ for $i \le s$; if this process continues indefinitely, we will set $A_i = \bigcup_{s \ge i} A_i[s]$ for every $i$. At each step of the construction, we will \textit{focus} on a particular $A_i$, and append an element to it; since we want each $A_i$ to be infinite, we will have to focus on every $A_i$ infinitely often. One way to do this is the following: at step $s$, focus on $i=d(s)$, where $d(s)$ is the element at index $s$ in the sequence 
    $$0,0,1,0,1,2,0,1,2,3,0,1,2,3,4,\dots$$
    
    We begin by initializing $H_0=G$, $A_0[0]=G$. Note that property \eqref{eqn:stage-building-property} is immediately satisfied for this initialization. Now suppose that $H_s$, $A_i[s]$ for $i \le s$ satisfying \eqref{eqn:stage-building-property} have been built. Consider the sets $A_{d(s)}[s]$ and $H_{s}$, and let $h=\max(H_{s})$. Let $S \subseteq \N_{> h}$ be any finite, non-empty set, for which, denoting $\max(S)=u$, it holds that
    \begin{align}
        \label{eqn:inductive-step-stage-building}
        f(A_{d(s)}[s] \cup \{u\}) = f(H_s \cup S).
    \end{align}
    If such a set $S$ exists, we will continue the construction as follows: we set $H_{s+1}=H_s \cup S$ and $A_{s+1}[s+1]=H_{s+1}$. For $i < s+1$ where $i \neq d(s)$, we set $A_i[s+1] = A_i[s]$. Finally, for $i = d(s)$, we set $A_i[s+1]=A_i[s] \cup \{u\}$. 

    Note that since $A_{s+1}[s+1]=H_{s+1}$, the property \eqref{eqn:stage-building-property} for $i=s+1$ is immediately satisfied. Now consider $i=d(s)$, for which $A_i[s+1]=A_i[s] \cup \{u\}$. Recall also that $H_{s+1}=H_s \cup S$, where $S \subseteq \N_{> \max(H_s)}$. Then, $A_i[s+1] \subseteq H_{s+1}$; additionally, for any $x \in A_i[s+1]$, where $x \neq u$, we have that $f(A_i[s+1]_{\le x})=f(A_i[s]_{\le x})$, which, by assumption, is equal to $f(H_{s, \le x})$, which is furthermore equal to $f(H_{s+1, \le x})$. As for $x=u$, we also have, by \eqref{eqn:inductive-step-stage-building}, that $f(A_i[s+1]_{\le x})=f(A_i[s] \cup \{u\})=f(H_s \cup S)=f(H_{s+1, \le x})$. Thus, property \eqref{eqn:stage-building-property} holds for $i=d(s)$. Finally, for $i \neq d(s), i \neq s+1$, since $A_i[s+1]=A_i[s]$, and elements added to $H_s$ are all above $\max(H_s)$, property \eqref{eqn:stage-building-property} holds simply by the previously maintained invariant. Thus, the construction validly proceeds forward.

    Now consider the case that the construction goes on indefinitely. Then, we will set $B= \bigcup_{s \in \N} H_s$, and $A_i = \bigcup_{s \ge i} A_i[s]$. Since at each step of the construction, we append at least an element to $H_s$, we have that $B$ is infinite. Furthermore, since for every $i$, $d(s)=i$ for infinitely many $s$, and we append an element to $A_{i}[s]$ at each such $s$, we also have that $A_i$ is infinite.

    We now argue that $A_i \subsetneq B$ for every $i$; to see this, observe that at infinitely many steps $s$ of the construction, we do \textit{not} focus on $A_i$. At each such step, we append an element to $H_s$ (and hence to $B$), which can never later be added to $A_i$, since all later additions to $A_i$ are larger than $\max(H_s)$; thus, $A_i \subsetneq B$. 

    Now fix any finite, non-empty $T \subseteq B$. By construction, any such $T$ must be contained in $H_s$ for some $s \ge 0$. Now recall that $A_s[s]$ is initialized to $H_s$, and thereafter, $A_s$ only grows. Thus, $T \subseteq A_s \subsetneq B$. It remains to argue that for every $x \in A_s$, $f(A_{s, \le x})=f(B_{\le x})$. Observe that by virtue of the initialization $A_s[s]=H_s$, and the fact that every subsequent element added to both $A_s$ and $B$ is larger than $\max(H_s)$, the desired equality holds for every $x \in H_s$. Now consider any element $x > \max(H_s)$ in $A_s$ --- this element must have been added to $A_s$ at some stage $q$, for which $d(q)=s$. Simultaneously, we would have incorporated a set $S$ into $B$, where $\max(S)=x$, and formed $H_{q+1}=H_q \cup S$. All subsequent additions to both $A_s$ and $B$ are entirely above $x$. Therefore, by the invariant maintained in \eqref{eqn:inductive-step-stage-building}, it is ensured that $f(A_{s, \le x})=f(A_{s}[q] \cup \{x\})=f(H_q \cup S) = f(B_{\le x})$ as desired. 
    
    We have thus argued that if the construction above goes on indefinitely, the inductive step in the proof of the lemma holds. Otherwise, the construction stops at some finite step $s$, by which time we will have built some finite set $H_s$; let $h=\max(H_s)$. Since the construction stopped, it must be the case that for any finite, non-empty $S \subseteq \N_{> h}$, denoting $u=\max(S)$, it holds that
    \begin{align*}
        f(A_{d(s)}[s] \cup \{u\}) \neq f(H_s \cup S).
    \end{align*}
    But now, observe that $A_{d(s)}[s] \cup \{u\} = G \cup F_1$ and $H_s \cup S = G \cup F_2$ for some finite sets $F_1,F_2$ that satisfy $\max(F_1)=\max(F_2)=u$. Since $G$ is $t$-coherent with respect to $f$, we have that both $f(A_{d(s)}[s] \cup \{u\})$ and $f(H_s \cup S)$ belong to some common palette $P_u$, where $|P_u| \le t$. But since $f(A_{d(s)}[s] \cup \{u\}) \neq f(H_s \cup S)$, we conclude that for every $u > h = \max(H_s)$, and for every finite $S \subseteq \N_{> h}$ satisfying $\max(S)=u$, it holds that $f(H_s \cup S) \in P'_u$, where $|P'_u| \le t-1$. Thus, we can set $G'=H_s \supseteq G$, which renders $G'$ to be $(t-1)$-coherent with respect to $f$, completing the inductive proof of the lemma.
\end{proof}

\subsection{Borel coloring functions}
\label{sec:borel-lb}

We now significantly generalize the class of coloring functions from finite-prefix coloring functions to \textit{Borel} coloring functions.
We will first precisely define the notion of Borel coloring functions that we consider. Let $S$ be any discrete (finite or countably infinite) set, and let $\mcZ_S$ denote the set of infinite strings whose characters belong to $S$, i.e.,
\begin{align*}
    \mcZ_S=\{(z_0,z_1,\dots):z_i \in S\; \forall i \in \N\}.
\end{align*}
We will consider the standard product topology over $\mcZ_S$, where the basic open sets are given by finite prefixes $p=(p_0,\dots,p_{n-1}) \in S^*$:
\begin{align}
    \label{eqn:open-set-form}
    \mcO_p = \{X \in \mcZ_S: X_0=p_0,\dots,X_{n-1}=p_{n-1}\}.
\end{align}
That is, the open set $\mcO_p$ comprises of all the strings in $\mcZ_S$ that begin with $p$. Then, the Borel subsets of $\mcZ_S$ are the elements of the $\sigma$-algebra generated by the open sets $\mcO_p$ --- namely, the smallest collection of subsets of $\mcZ_S$, which contains all the basic open sets $\mcO_p$ for every $p \in S^*$, and is closed under complements, countable unions and countable intersections.

We will now introduce some additional notation: for any $S \subseteq \N$, where $|S|=|\omega|$, let $[S]^\omega = \{L \subseteq S: |L|=|\omega|\}$. %
Similar to the above, we will consider the standard product topology over $[S]^\omega$, whose basic open sets are given by the sets 
\begin{align}
    \label{eqn:open-sets-languages}
    \{x_0 < x_1 < \dots \in [S]^\omega\mid x_0 = p_0,\dots,x_{n-1}=p_{n-1}\}
\end{align}
for every finite prefix $p_0 < \dots < p_{n-1}$. The Borel subsets of $[S]^\omega$ are then the elements of the $\sigma$-algebra generated by these open sets.

Note that a language is an element of $[\N]^\omega$. We will be concerned with coloring functions that map the strings in a language to a sequence of colors in a ``Borel'' manner, i.e., as given by a Borel map $J$ mapping $[\N]^\omega$ to $\mcZ_\palette$, where $\palette$ is a finite palette having size $k$; for convenience, we will assume $\palette=\{1,2,\dots,k\}$. That is, $J$ has the property that $J^{-1}(Y)$ is Borel subset of $[\N]^\omega$, for every Borel subset $Y \subseteq \mcZ_\palette$.

We now formally define the family of Borel coloring functions:

\begin{definition}[Borel coloring functions]
    Let $\mcC=[\N]^\omega$. A family of Borel coloring functions $\{c^J_L\}_{L \in \mcC}$ is specified by a Borel map $J:[\N]^\omega \to \mcZ_\palette$, which satisfies that, for any language $L = \{x_0 < x_1 < \dots\}$,
    \begin{align*}
        c^J_L(x_i) = J(L)_i.
    \end{align*}
\end{definition}

We will show that Borel coloring functions that use a finite palette do not suffice for identification. A key tool that we will use is the following famous result by Galvin and Prikry \cite{galvinprikry} which shows that every Borel set is Ramsey.\footnote{The result in \cite{galvinprikry} is originally stated for $Z=\N$; the form that we state follows simply by identifying every element in $\N$ with an element in $Z$.}

\begin{theorem}[Galvin-Prikry \cite{galvinprikry}]
    \label{thm:galvin-prikry}
    Let $Z \in [\N]^\omega$, and let $R \subseteq [Z]^\omega$ be a Borel subset of $[Z]^\omega$. Then, there exists $Y \in [Z]^\omega$ such that either $[Y]^\omega \subseteq R$ or $[Y]^\omega \cap R = \emptyset$.
\end{theorem}

We now show the following \textit{canonicalization} lemma, which uses the Galvin-Prikry theorem above, and shows that Borel coloring functions, when restricted to particular languages in $[\N]^\omega$, behave like finite-prefix coloring functions.

\begin{lemma}[Canonicalization Lemma]
    \label{lemma:canonicalization}
    Let $\mcC$ be the collection of all infinite subsets of $\N$, and let $\{c^J_L\}_{L \in \mcC}$ be a family of Borel coloring functions specified by a Borel map $J$ over a finite palette $\palette$. There exists an infinite set $X=\{x_0 < x_1 < \dots\} \in [\N]^\omega$ and a finite-prefix function $f$ that maps finite, non-empty subsets of $X$ to the palette $\palette$, such that for every finite, non-empty $S \subseteq X$, where $S=\{x_{i_0} < \dots < x_{i_{m}}\}$, and every infinite tail $T \in [X_{> x_{i_{m}}}]^\omega$, it holds that
    \begin{align*}
        J(S \cup T)_{m} = f(S). %
    \end{align*}
\end{lemma}
In words, the lemma above identifies a special subsequence in $[\N]^\omega$, such that, if we take take any finite subset $S$ of the subsequence, and then append it with an infinite tail $T$ above $S$, then the color assigned by the Borel coloring function $J$ (which acts on all of $S \cup T$) to the largest member of $S$ is independent of the tail $T$, and is specified completely by a finite-prefix function acting only on the set $S$.
\begin{proof}
    We will specify the elements of the set $X$ inductively. At step $n$ of the process, suppose we have constructed $X_n=\{x_0 < x_1 < \dots < x_{n-1}\}$ and a ``tail reservoir'' $Z_n \in [\N_{> x_{n-1}}]^\omega$, such that for every non-empty subset $S=\{x_{i_0} < \dots < x_{i_{m}}\} \subseteq X_n$, $f(S)$ has been defined, and satisfies that, for every infinite tail $T \in [Z_n]^\omega$,
    \begin{align*}
        J(S \cup T)_{m} = f(S).
    \end{align*}
    For the base case $n=0$, we simply have that $X_0=\{\}$, $Z_0=\N$, and the claim holds vacuously, since there are no non-empty subsets of $X_0$.

    Now suppose that we have constructed $X_n$, $Z_n$, and that $f(S)$ has been defined for every non-empty subset $S \subseteq X_n$. We will set $x_n=\min(Z_n)$, so that $X_{n+1}=\{x_0 < \dots < x_{n-1} < x_n\}$. We must now define the tail reservoir $Z_{n+1}$, and also $f(S)$ for every non-empty subset $S \subseteq\{x_0<\dots<x_n\}$ that contains $x_n$, and thereafter verify that the required property is satisfied with respect to the updated tail reservoir $Z_{n+1}$.

    Let us enumerate all the finitely many non-empty subsets of $X_{n+1}$ that contain $x_n$ as $S_1,\dots,S_\ell$ (where $\ell=2^{n}$), and let $|S_j|=m_j+1$ (recall: we use 0-indexing). We will process each $S_j$ in order; furthermore, for each $S_j$, we will consider each color $c \in \{1,2,\dots,k\}$ in order. As we run through these, we will maintain a series of tail reservoirs 
    \begin{align*}
        Q_{1,0},\dots,Q_{1,k-1},Q_{2,0},\dots,Q_{2,k-1},\dots,Q_{\ell, 0},\dots,Q_{\ell, k-1}.
    \end{align*}
    These tail reservoirs will have the property that $Q_{j,c+1} \subseteq Q_{j,c}$, and $Q_{j+1,0} \subseteq Q_{j,k-1}$.

    To begin with, let $Q_{0,k-1}=Z_{n, > x_n}$. We will now process $S_1,\dots,S_{\ell}$ in order. $(\star)$ Suppose we begin to process $S_j$. We initialize $Q_{j,0}=Q_{j-1, k-1}$. Now, $(\star \star)$ suppose we are processing color $c \in \{1,\dots,k\}$ for $S_j$. Define the following subset of $[Q_{j,c-1}]^\omega$:
    \begin{align}
        \label{eqn:borel-subset-for-galvin-prikry}
        R_{c,j} := \left\{T \in [Q_{j,c-1}]^\omega ~\Big |~ J(S_j \cup T)_{m_j}=c\right\}.
    \end{align}
    \Cref{claim:borel-subset-for-galvin-prikry} at the end of the proof shows that $R_{c,j}$ is a Borel subset of $[Q_{j,c-1}]^\omega$. By the Galvin-Prikry theorem (\Cref{thm:galvin-prikry}), there exists $Y \in [Q_{j,c-1}]^\omega$ such that either $[Y]^\omega \subseteq R_{c,j}$ or $[Y]^\omega \cap R_{c,j} = \emptyset$. 
    
    In the former case, we define $f(S_j)=c$, set $Q_{j,c}=Q_{j,c+1}=\dots=Q_{j,k-1}=Y$, and go back to $(\star)$ to process the next set $S_{j+1}$. Note that this ensures: for every $T \in [Q_{j,k-1}]^\omega$, $J(S_j \cup T)_{m_j}=c$.

    In the latter case, we set $Q_{j,c}=Y$; we remain at the set $S_j$, but go back to $(\star \star)$ to process the next color $c+1$. Importantly, note that in this case, we have reduced a color: namely, every $T \in [Q_{j,c}]^\omega$ satisfies that $J(S_j \cup T)_{m_j} \neq c$; so, if we have processed colors $1,\dots,c$ while remaining in this case, we have that $J(S_j \cup T)_{m_j}$ is restricted to exactly be one of the colors $\{c+1,\dots,k\}$. This means that if we continue to be in this case for each of $c=1,\dots,k-1$, when we begin processing $c=k$, we have the property that every $T \in [Q_{j,k-1}]^\omega$ satisfies the property that $J(S_j \cup T)_{m_j}=k$. We then set $f(S_j)=k$, and go back to $(\star)$ to process the next set $S_{j+1}$.

    After we have processed each of $S_1,\dots,S_\ell$, we will have defined $f(S)$ for every non-empty finite subset of $X_{n+1}$. To finish up, we set $Z_{n+1}=Q_{\ell, k-1}$.

    To complete the inductive step, it remains to argue that for every non-empty, finite subset $S \subseteq \{x_0,\dots,x_n\}$ where $|S|=m+1$, and for every tail $T \in [Z_{n+1}]^\omega$, $J(S \cup T)_m = f(S)$. If $S \subseteq \{x_0,\dots,x_{n-1}\}$, this holds simply by the inductive hypothesis, and the observation that $Z_{n+1} \subseteq Z_n$. Otherwise, $x_n \in S$; in this case, $S=S_j$ for some $S_j$ above. Then, observe that after we process $S_j$, $f(S_j)$ is set to be equal to precisely that $c$ value for which it holds that $J(S_j \cup T)_{m}=c$ for every $T \in [Q_{j,k-1}]^\omega$; since $Z_{n+1} \subseteq Q_{j, k-1}$, this property also holds for every $T \in [Z_{n+1}]^\omega$. This completes the inductive construction of the elements in $X$, and also defines the function $f$ at every finite, non-empty subset of $X$.

    Finally, we argue that for every finite, non-empty $S \subseteq X$, where $S=\{x_{i_0} < \dots < x_{i_{m}}\}$, and every $T \in [X_{> x_{i_{m}}}]^\omega$, it holds that $J(S \cup T)_{m}=f(S)$. To see this, let $i_{m}=n$. Then, at the time $x_n$ was added to $X$ in the inductive construction above, we ensured that for every finite, non-empty subset $S$ of $\{x_0,\dots,x_n\}$, it holds that $J(S \cup T)_{|S|-1}=f(S)$ for every $T \in [Z_{n+1}]^\omega$. In particular, this holds for $S=\{x_{i_0} < \dots < x_{i_{m}}\}$. The desired conclusion then follows by noting that $X_{> x_{i_{m}}}=X_{> x_n} \subseteq Z_{n+1}$.

    We now prove the promised claim that the sets $R_{c,j}$ defined in \eqref{eqn:borel-subset-for-galvin-prikry} are Borel.
    \begin{claim}
        \label{claim:borel-subset-for-galvin-prikry}
        The set $R_{c,j}$ defined in \eqref{eqn:borel-subset-for-galvin-prikry} is a Borel subset of $[Q_{j, c-1}]^\omega$.
    \end{claim}
    \begin{proof}
        Recall that we defined $R_{c,j}$ while processing the finite set $S_j$, where $|S_j|=m_j+1$. Consider the following subset of $\mcZ_\palette$:
        \begin{align*}
            I = \{C \in \mcZ_\palette ~\Big |~ C_{m_j}=c\}.
        \end{align*}
        Note that $I$ is a finite union of open sets of the form in \eqref{eqn:open-set-form}, and is hence an open set in the product topology over $\mcZ_\palette$. Since $J:[\N]^\omega \to \mcZ_\palette$ is a Borel map, we have that $J^{-1}(I)$ is a Borel subset of $[\N]^\omega$. That is,
        \begin{align*}
            A = J^{-1}(I)= \{A \in [\N]^\omega ~\Big |~ J(A)_{m_j}=c\}
        \end{align*}
        is a Borel subset of $[\N]^\omega$. Now consider the map
        \begin{align*}
            \varphi:[Q_{j,c-1}]^\omega \to [\N]^\omega, \qquad \varphi(Y)=S_j \cup Y.
        \end{align*}
        Note that $\varphi$ is continuous --- for any basic open set $\mcO_p=\{X \in [\N]^\omega: X_0=p_0,\dots,X_{n-1}=p_{n-1}\}$ corresponding to the finite prefix $p=\{p_0<\dots<p_{n-1}\}$, we have that $\varphi^{-1}(\mcO_p)$ is either equal to $\emptyset$, $[Q_{j,c-1}]^\omega$ or a basic open set in $[Q_{j,c-1}]^\omega$, depending on the consistency of the prefix $p$ with $S_j$. Finally, $\varphi$ being continuous means that
        \begin{align*}
            \varphi^{-1}(A) = \left\{T \in [Q_{j,c-1}]^\omega ~\Big |~ J(S_j \cup T)_{m_j}=c\right\}=R_{c,j}
        \end{align*}
        is a Borel subset of $[Q_{j,c-1}]^\omega$ as claimed.
    \end{proof}
    This concludes the proof of \Cref{lemma:canonicalization}.
\end{proof}

Using the canonicalization lemma above, we now show that the class of Borel coloring functions do not suffice for identification with a finite palette.

\begin{theorem}
    \label{thm:borel-lb-identification}
    Let $\mcC$ be the collection of all infinite subsets of $\N$, and let $\{c^J_L\}_{L \in \mcC}$ be any universal family of Borel coloring functions specified by the Borel map $J$, and mapping to a finite palette $\palette$. There exists a countable collection $\mcC'$ such that $\mcC'$ is not identifiable in the limit, even with terminal colors given by $\{c^J_L\}_{L \in \mcC'}$.
\end{theorem}
\begin{proof}
    From \Cref{lemma:canonicalization}, there exists an infinite set $X=\{x_0 < x_1 < \dots\} \in [\N]^\omega$ and a finite-prefix function $f$ that maps finite, non-empty subsets of $X$ to the palette $\palette$, such that for every finite, non-empty $S \subseteq X$, where $S=\{x_{i_0} < \dots < x_{i_{m}}\}$, and every infinite tail $T \in [X_{> x_{i_{m}}}]^\omega$, it holds that
    \begin{align}
        \label{eqn:canonicalization-condition}
        J(S \cup T)_{m} = f(S). %
    \end{align}
    Fix a bijection $\pi:\N \to X$ satisfying that $\pi(n)=x_n$. For any $A \subseteq \N$, let us slightly abuse notation and use the shorthand
    \begin{align*}
        \pi(A) := \{\pi(a):a \in A\} = \{x_{a} : a \in A\} \subseteq X.
    \end{align*}
    Now define the function $g$ that maps finite (ordered) subsets of $\N$ to the palette $\palette$, where
    \begin{align*}
        g(S) = f(\pi(S)).
    \end{align*}
    Consider now the universal family $\{c^g_L\}_{L \in \mcC}$ of finite-prefix coloring functions, specified by the function $g$, which satisfies that, for any language $L=\{y_0 < y_1 < \dots\} \in [\N]^\omega$,
    \begin{align*}
        c^g_L(y_{i}) = g(\{y_0,\dots,y_i\}).
    \end{align*}
    Observe that 
    \begin{align}
        c^g_L(y_i) = g(\{y_0,\dots,y_i\}) &= f(\pi(\{y_0,\dots,y_i\})) \nonumber \\
        &= J(\{\pi(y_0),\dots,\pi(y_i)\} \cup \{\pi(y_{i+1}), \pi(y_{i+2}), \dots\})_{i} \tag{using \eqref{eqn:canonicalization-condition}} \nonumber \\
        &= J(\pi(L))_{i}. \label{eqn:relating-f-and-g}
    \end{align}

    Now assume for the sake of contradiction that for every countable collection $\mcC'$ comprising of languages from $\mcC$, $\mcC'$ is identifiable in the limit with terminal colors given by the Borel coloring functions $\{c^J_L\}_{L \in \mcC'}$. We will then show that for every countable collection $\mcC'$ comprising of languages from $\mcC$, $\mcC'$ is \textit{also} identifiable in the limit with terminal colors given by the finite-prefix coloring functions $\{c^g_L\}_{L \in \mcC'}$, which will contradict \Cref{thm:finite-prefix-lb-identification}.
    
    Towards this, fix any countable subcollection $\mcC' \subseteq \mcC$. %
    We will show that for every $L \in \mcC'$, there exist finite $T_L \subseteq L$ satisfying
    \begin{align}
        \label{eqn:to-show-for-id-with-finite-prefix}
        \forall L' \in \mcC' \text{ if } L' \subsetneq L \text{ then } \left[(T_{L} \nsubseteq L') \; \lor \; (\exists x \in L'  \text{ such that } c^g_{L'}(x) \neq c^g_{L}(x)) \right].
    \end{align}
    From \Cref{thm:characterization-identification-with-terminal-traces}, this implies that $\mcC'$ is identifiable in the limit with terminal colors given by $\{c^g_L\}_{L \in \mcC'}$.

    In order to define the tell-tales $T_L$, let us consider the collection $\mcC'' = \{\pi(L)\}_{L \in \mcC'}$. By assumption, $\mcC''$ is identifiable in the limit with terminal colors given by the Borel map $J$. From \Cref{thm:characterization-identification-with-terminal-traces}, this means that for every $\pi(L) \in \mcC''$, there exist finite $\widetilde{T}_{\pi(L)} \subseteq \pi(L)$ satisfying
    \begin{align}
        \label{eqn:id-that-holds-true-for-borel}
        \forall \pi(L') \in \mcC'' \text{ if } \pi(L') \subsetneq \pi(L) \text{ then } \left[(\widetilde{T}_{\pi(L)} \nsubseteq \pi(L')) \; \lor \; (\exists z \in \pi(L')  \text{ such that } c^J_{\pi(L')}(z) \neq c^J_{\pi(L)}(z)) \right].
    \end{align}
    For any $L \in \mcC'$, we now define $T_L \subseteq L$ as follows:
    \begin{align}
        \label{eqn:def-transferred-tell-tale}
        T_L = \pi^{-1}\left({\widetilde{T}_{\pi(L)}}\right).
    \end{align}
    It remains to show that \eqref{eqn:to-show-for-id-with-finite-prefix} holds true for $T_L$ defined as such. Towards this, consider any $L' \in \mcC'$ that satisfies $L' \subsetneq L$. Now consider the languages $\pi(L), \pi(L') \in \mcC''$. Since $\pi$ is a bijection into $X$, $L' \subsetneq L$ implies that $\pi(L') \subsetneq \pi(L)$. From \eqref{eqn:id-that-holds-true-for-borel}, this further means that 
    \begin{align*}
        (\widetilde{T}_{\pi(L)} \nsubseteq \pi(L')) \; \lor \; (\exists z \in \pi(L')  \text{ such that } c^J_{\pi(L')}(z) \neq c^J_{\pi(L)}(z))
    \end{align*}
    If $\widetilde{T}_{\pi(L)} \nsubseteq \pi(L')$,  we have, using the definition of $T_L$ \eqref{eqn:def-transferred-tell-tale} and the fact that $\pi$ is a bijection, that
    \begin{align*}
        \widetilde{T}_{\pi(L)} \nsubseteq \pi(L') \implies \pi(T_L) \nsubseteq \pi(L') \implies T_L \nsubseteq L'.
    \end{align*}
    On the other hand, suppose $\exists z \in \pi(L')  \text{ such that } c^J_{\pi(L')}(z) \neq c^J_{\pi(L)}(z)$. Suppose that $L'=\{y'_0 < y'_1 < \dots\} \in [\N]^\omega$, so that $\pi(L')=\{\pi(y'_0) < \pi(y'_1) < \dots\}$, and let $z=\pi(y'_i)$. Then, we have that
    \begin{align*}
        c^J_{\pi(L')}(z) &= c^J_{\pi(L')}(\pi(y'_i)) \\
        &= J(\pi(L'))_{i} \\
        &= f(\pi(\{y'_0,\dots,y'_i\})) \tag{using \eqref{eqn:relating-f-and-g}}\\
        &= c^g_{L'}(y'_i). \tag{by definition of $c^g_{L'}$}
    \end{align*}
    Now suppose that $L=\{y_0 < y_1 < \dots\} \in [\N]^\omega$, so that $\pi(L)=\{\pi(y_0) < \pi(y_1) < \dots\}$; since $L' \subsetneq L$, we have that $\pi(L') \subsetneq \pi(L)$, and so, $z \in \pi(L)$. Then, let $z=\pi(y_j)$. Note that $y_j=y'_i=\pi^{-1}(z)$, since $\pi$ is a bijection. By a similar chain of equalities as above, it holds that
    \begin{align*}
        c^J_{\pi(L)}(z) = c^g_L(y_j) = c^g_L(y'_i).
    \end{align*}
    Therefore, since $c^J_{\pi(L')}(z) \neq c^J_{\pi(L)}(z)$, we have that $c^g_{L'}(y'_i) \neq c^g_L(y'_i)$. That is, we have shown the existence of $x \in L'$ such that $c^g_{L'}(x) \neq c^g_L(x)$. Combining both the cases above establishes \eqref{eqn:to-show-for-id-with-finite-prefix} for the tell-tales defined in \eqref{eqn:def-transferred-tell-tale}. This implies that $\mcC'$ is identifiable in the limit with terminal colors given by $\{c^g_L\}_{L \in \mcC'}$. Since $\mcC'$ above was arbitrary, we have shown that for every countable collection $\mcC'$ comprising of languages from $\mcC$, $\mcC'$ is identifiable in the limit with terminal colors given by the finite-prefix coloring functions $\{c^g_L\}_{L \in \mcC'}$. This contradicts \Cref{thm:finite-prefix-lb-identification}, yielding the desired result.
\end{proof}

\begin{remark}[Finite Languages III]
    \label{remark:finite-languages-3}
    Our lower bounds above rule out collection-independent Borel coloring functions and finite-prefix coloring functions that use a finite palette for identification of all countable language collections, where every language is infinite. The lower bound continues to hold when we additionally also include finite languages, since this only makes the problem harder.
\end{remark}

\subsection{Existence of Borel coloring functions that use an infinite palette}
\label{sec:borel-ub}

The previous subsection showed that no family of Borel coloring functions that uses a finite palette can work in a collection-independent manner for identification. We conclude this section by showing that there exists a family of Borel coloring functions that uses an \textit{infinite} palette and suffices for identification collection-independently.

\begin{theorem}
    \label{thm:borel-map-infinite-palette}
    Let $\mcC$ be the collection of all infinite subsets of $\Sigma^*$. There exists a family of Borel coloring functions $\{c^J_{L}\}_{L \in \mcC}$ given by a Borel map $J:\mcC \to \mcZ_\N$ that satisfies the distinguishable coloring condition.
\end{theorem}
\begin{proof}
    The coloring functions are based on the trace coloring functions given by \cite{charikar26language}: we show that the natural conversion of these trace coloring functions into terminal coloring functions --- which uses an infinite palette --- satisfies the Borel property, and also satisfies the distinguishable coloring condition.
    
    In more detail, the trace coloring functions of \cite{charikar26language} are defined as follows. 
    For any $L \subseteq \Sigma^*$, define the function $f_L:\Sigma^* \to \{0,1,\dots,|\Sigma|+1\}$ as
    \begin{align}
        \label{eqn:nextlive-trace-coloring-1}
        f_L(y) = \Ind[y \in L] + \sum_{a \in \Sigma} \Ind[\exists q \in \Sigma^* \text{ such that } yaq \in L].
    \end{align}
    Then, the trace coloring of any $x=(s_{0},\dots,s_{n}) \in L$ is the sequence
    \begin{align}
        \label{eqn:next-live-trace-coloring-2}
        (f_L(\eps), f_L(s_{\le 0}), f_L(s_{\le 1}),\dots,f_L(s_{\le n})),
    \end{align}
    where $s_{\le j}=(s_{0},\dots,s_{j})$. We can naturally convert this trace coloring into a terminal coloring by introducing a bijection mapping finite, non-empty sequences over $\{0,1,\dots,|\Sigma|+1\}$ to $\N$. Concretely, fix any bijection $b:\{0,1,\dots,|\Sigma|+1\}^{> 0} \to \N$. Now define the terminal coloring function $c_L$ as
    \begin{align}
        \label{eqn:nextlive-terminal-coloring}
        c_L(x) = b((f_L(\eps), f_L(s_{\le 0}), f_L(s_{\le 1}),\dots,f_L(s_{\le n}))).
    \end{align}
    Note that this terminal coloring function uses an infinite palette. \cite{charikar26language} showed that the trace coloring functions given by \eqref{eqn:nextlive-trace-coloring-1}, \eqref{eqn:next-live-trace-coloring-2} satisfy the distinguishable coloring condition (over the collection $\mcC=[\Sigma^*]^\omega$)\footnote{In fact, \cite{charikar26language} show that distinguishable coloring condition holds for $2^{\Sigma^*}$, which includes finite subsets of $\Sigma^*$ as well.}, where instead of terminal colors, they consider trace colors. Since $b$ induces a bijection between the two, whenever there is a discrepancy between the trace coloring of a string, there is a discrepancy between the terminal coloring as well. Summarily, it holds that the family of terminal colorings given in \eqref{eqn:nextlive-terminal-coloring} satisfies the distinguishable coloring condition.
    
    Now, fix a canonical ordering of all the strings in $\Sigma^*$, and suppose any $L=\{x_0,x_1,\dots\} \in [\Sigma^*]^\omega$ is specified in this ordering. Consider the map $J:[\Sigma^*]^\omega \to \mcZ_\N$ to be defined as
    \begin{align*}
        J(L) = (c_L(x_0), c_L(x_1),\dots).
    \end{align*}
    The following claim shows that this map is Borel, which completes the proof of \Cref{thm:borel-map-infinite-palette}.
    
    \begin{claim}
        \label{claim:nextlive-borel-map}
        The terminal coloring functions defined in \eqref{eqn:nextlive-terminal-coloring} induce a Borel map $J$.
    \end{claim}
    \begin{proof}
        For any finite $p \in \Sigma^*$, define $F_p:[\Sigma^*]^\omega \to \{0,1,\dots,|\Sigma|+1\}$ as %
        \begin{align*}
            F_p(L) = \Ind[p \in L] + \sum_{a \in \Sigma} \Ind[\exists q \in \Sigma^* \text{ such that } paq \in L].
        \end{align*}
        Namely, $F_p$ is defined similarly to $f_L$ in \eqref{eqn:nextlive-trace-coloring-1}, but swaps arguments. We will argue that $F_p$ is a Borel map.
    
        Towards this, note first that the map $L \mapsto \Ind[p \in L]$ is a Borel map. This is because the set $\{L: p \in L\}$, which can be written as a countable union of open sets of the form given in \eqref{eqn:open-sets-languages}, is a Borel subset of $[\Sigma^*]^\omega$.
    
        Next, we claim that for any $a \in \Sigma$, the map $L \mapsto \Ind[\exists q \in \Sigma^* \text{ such that } paq \in L]$ is a Borel map. This follows since
        \begin{align*}
            \{L: \exists q \in \Sigma^* \text{ such that } paq \in L\} = \bigcup_{q \in \Sigma^*} \{L: paq \in L\},
        \end{align*}
        and each $\{L: paq \in L\}$ can be written as a countable union of open sets. Since $\Sigma$ is finite, we get that the map $L \mapsto \sum_{a \in \Sigma}\Ind[\exists q \in \Sigma^* \text{ such that } paq \in L]$ is Borel. In total, we conclude that $F_p$ is Borel.
    
        Now consider the \textit{selector} function $S_i:[\Sigma^*]^\omega \to \Sigma^*$ which selects the $i^\text{th}$ element in the enumeration of any language $L=\{x_0, x_1, \dots\}$, i.e., $S_i(L)=x_i$. We argue that $S_i$ is Borel. Towards this, consider any $\sigma_m$ in the canonical ordering of $\Sigma^*=\{\sigma_0,\sigma_1,\dots\}$. It suffices to show that $S_i^{-1}(\sigma_m)$ is a Borel subset of $[\Sigma^*]$. Since all languages $L$ are enumerated in an order consistent with the canonical ordering of $\Sigma^*$, we have that
        \begin{align*}
            S_i^{-1}(\sigma_m) = \bigcup_{j_0<j_1<\dots<j_{i-1} < m} \left\{L=\{x_0 < x_1 < \dots\} \in [\Sigma^*]^\omega ~\Big |~ x_0 = \sigma_{j_0},\dots,x_{i-1}=\sigma_{j_{i-1}}, x_i = \sigma_m\right\},
        \end{align*}
        which is a finite union of open subsets of $[\Sigma^*]^\omega$, and is hence Borel.
    
        We will now show that the coordinate-wise map $J_i:[\Sigma^*]^\omega \to \N$ where $J_i(L)=J(L)_i$ is a Borel map. In order for this, it suffices to show that $J_i^{-1}(n)$ for any given $n \in \N$ is a Borel set. To this end, recall that $b$ is bijection mapping finite sequences over $\{0,1,\dots,|\Sigma|+1\}$ to $\N$, and consider $b^{-1}(n) = (r_0,\dots,r_{\ell-1}) \in \{0,1,\dots,|\Sigma|+1\}^\ell$. Then, we have that $J_i(L)=n$ if and only if $L=\{x_0 < x_1 < \dots\}$ satisfies that $|x_i|=\ell-1$, $F_{\eps}(L)=r_0$, and for every $j=0,\dots,\ell-2$, $F_{x_{i, \le j}}(L)=r_{j+1}$. Namely,
        \begin{align*}
            J_i^{-1}(n) = \bigcup_{s \in \Sigma^{\ell-1}} \left( S^{-1}_i(s) \cap F^{-1}_{\eps}(r_0) \, \cap \, \bigcap_{j=0}^{\ell-2} F^{-1}_{s_{\le j}}(r_{j+1})\right).
        \end{align*}
        The expression inside the (finite) union is a finite intersection of Borel sets (and is hence a Borel set), since we argued above that both $S_i$ and $F_p$ are Borel maps. We conclude that $J^{-1}_i(n)$ is a Borel subset of $[\Sigma^*]^\omega$; since $n$ was arbitrary, this implies that $J_i$ is a Borel map. Finally, since $\mcZ_\N$ has the product Borel structure over $\N$, this suffices to show that $J:[\Sigma^*]^\omega \to \mcZ_\N$ is Borel, as desired.
    \end{proof}
\end{proof}

\paragraph{AI Disclosure} ChatGPT was routinely used to help in developing and understanding the technical ideas used in the paper. In particular, it pointed us to the connection with the Galvin-Prikry theorem used in our lower bounds.

\subsection*{Acknowledgements}
This work was supported by Moses Charikar’s and Gregory Valiant’s Simons Investigator Awards, a Google PhD Fellowship, a Simons Collaboration grant and a grant from the MacArthur Foundation.

\bibliographystyle{alpha} 
\bibliography{references}

\appendix

\end{document}